\documentclass[11pt]{article}  

\usepackage{etex}

\usepackage{fullpage}

\usepackage{algorithm}
\usepackage{algorithmicx,algpseudocode}
\usepackage{amsmath}
\usepackage{amssymb}
\usepackage{amsthm}
\usepackage{array}
\usepackage{bbm}
\usepackage{cancel}
\usepackage{cmap}
\usepackage{enumerate}
\usepackage{enumitem}
\usepackage{fancyhdr}
\usepackage{mathdots}
\usepackage{mathtools}
\usepackage{mathrsfs}
\usepackage{hyperref}
\usepackage{stackrel}
\usepackage{stmaryrd}
\usepackage{tabularx}
\usepackage{ctable}
\usepackage{titlesec}
\usepackage{titletoc}
\usepackage{url}
\usepackage{verbatim}
\usepackage{wasysym}
\usepackage{wrapfig}
\usepackage{yhmath}
\usepackage[all,cmtip]{xy}

\usepackage{hyperref}
\usepackage[%
	firstinits=true,
	doi=false,
	url=false,
	isbn=false,
	backend=biber,
	style=alphabetic,hyperref]{biblatex}

\usepackage{algorithm}
\usepackage{algpseudocode}
\algnewcommand\algorithmicinput{\textbf{Input: }}
\algnewcommand\INPUT{\State\algorithmicinput}
\algnewcommand\algorithmicinitialize{\textbf{Initialize: }}
\algnewcommand\INIT{\State\algorithmicinitialize}
\algnewcommand\algorithmicrun{\textbf{Run: }}
\algnewcommand\RUN{\State\algorithmicrun}
\algnewcommand\algorithmicupdate{\textbf{Update: }}
\algnewcommand\UPDATE{\State\algorithmicupdate}
\algnewcommand\algorithmicset{\textbf{Set: }}
\algnewcommand\SET{\State\algorithmicset}
\algnewcommand\algorithmicquery{\textbf{Query: }}
\algnewcommand\QUERY{\State\algorithmicquery}
\algnewcommand\algorithmicoutput{\textbf{Output: }}
\algnewcommand\OUTPUT{\State\algorithmicoutput}

\newcommand{\ug}[1]{\textsf{\color{green} UG: #1}}

\def\norm#1{\mathopen\| #1 \mathclose\|}

\newcommand{\poly}{\mathop{\mbox{\rm poly}}}

\newcommand{\ignore}[1]{}

\def\bold0{\mathbf{0}}

\newcommand{\val}{\mathop{\mbox{\rm val}}}

\def\eps{\varepsilon}
\def\epsilon{\varepsilon}

\newcommand{\braces}[1]{\left\{#1\right\}}

\newcommand{\bra}[1]{\left[#1\right]}
\newcommand{\abs}[1]{\left|#1\right|}

\DeclareMathOperator*{\argmin}{argmin}

\newtheorem{theorem}{theorem}[section]
\newtheorem{thm}[theorem]{Theorem}

\newtheorem{asm}[theorem]{Assumption}
\newtheorem{cor}[theorem]{Corollary}

\newtheorem{df}[theorem]{Definition}
\newtheorem{definition}[theorem]{Definition}
\newtheorem{lem}[theorem]{Lemma}
\newtheorem{lemma}[theorem]{Lemma}
\newtheorem{proposition}[theorem]{Proposition}

\newtheorem*{rem*}{Remark}
\newtheorem{pr}[theorem]{Proposition}
\newtheorem*{pr*}{Proposition}

\theoremstyle{definition}

\DeclareMathOperator*\E{\mathbb{E}}
\newcommand\R{\mathbb{R}}

\newcommand\C{\mathbb{C}}

\newcommand\wh{\widehat}
\newcommand\wt{\widetilde}

\def\shownotes{1}  \ifnum\shownotes=1
\newcommand{\authnote}[2]{$\ll$\textsf{\footnotesize #1 notes: #2}$\gg$}
\else
\newcommand{\authnote}[2]{}
\fi

\newcommand{\cF}[0]{\mathcal{F}}

\newcommand{\N}[0]{\mathbb{N}}

\newcommand{\Pj}[0]{\mathbb{P}}

\newcommand{\al}[0]{\alpha}
\newcommand{\be}[0]{\beta}
\newcommand{\ga}[0]{\gamma}

\newcommand{\de}[0]{\delta}
\newcommand{\De}[0]{\Delta}
\newcommand{\ep}[0]{\varepsilon}

\newcommand{\la}[0]{\lambda}

\newcommand{\rh}[0]{\rho}

\newcommand{\Te}[0]{\Theta}

\newcommand{\Om}[0]{\Omega}
\newcommand{\si}[0]{\sigma}
\newcommand{\Si}[0]{\Sigma}
\newcommand{\ze}[0]{\zeta}

\newcommand{\nin}[0]{\not\in}
\newcommand{\opl}[0]{\oplus}

\newcommand{\ot}[0]{\otimes}

\newcommand{\sub}[0]{\subset}

\newcommand{\subeq}[0]{\subseteq}

\newcommand{\bs}[0]{\backslash}
\newcommand{\iy}[0]{\infty}

\newcommand{\rc}[1]{\frac{1}{#1}}
\newcommand{\prc}[1]{\pa{\rc{#1}}}
\newcommand{\ff}[2]{\left\lfloor\frac{#1}{#2}\right\rfloor}

\newcommand{\fc}[2]{\frac{#1}{#2}}

\newcommand{\pf}[2]{\pa{\frac{#1}{#2}}}

\newcommand{\ab}[1]{\left| {#1} \right|}
\newcommand{\an}[1]{\left\langle {#1}\right\rangle}
\newcommand{\ba}[1]{\left[ {#1} \right]}
\newcommand{\bc}[1]{\left\{ {#1} \right\}}

\newcommand{\ce}[1]{\left\lceil {#1}\right\rceil}
\newcommand{\fl}[1]{\left\lfloor {#1}\right\rfloor}

\newcommand{\pa}[1]{\left( {#1} \right)}

\newcommand{\ve}[1]{\left\Vert {#1}\right\Vert}

\newcommand{\set}[2]{\left\{{#1}:{#2}\right\}}

\newcommand{\ub}[2]{\underbrace{#1}_{#2}}

\newcommand{\amin}{\operatorname{argmin}}

\newcommand{\Tr}[0]{\operatorname{Tr}}

\newcommand{\Vol}[0]{\operatorname{Vol}}
\providecommand{\cal}[1]{\mathcal{#1}}
\renewcommand{\cal}[1]{\mathcal{#1}}

\newcommand{\pull}[9]{
#1\ar@/_/[ddr]_{#2} \ar@{.>}[rd]^{#3} \ar@/^/[rrd]^{#4} & &\\
& #5\ar[r]^{#6}\ar[d]^{#8} &#7\ar[d]^{#9} \\}

\newcommand{\cmp}[9]{
\xymatrix{
#1 \ar[r]^{#4}{#5} \ar@/_2pc/[rr]^{#8}_{#9} & #2 \ar[r]^{#6}_{#7} & #3
}
}

\newcommand{\ha}[1]{\ar@{^(->}[#1]}
\newcommand{\ls}[1]{\ar@{-}[#1]}
\newcommand{\sj}[1]{\ar@{->>}[#1]}
\newcommand{\aq}[1]{\ar@{=}[#1]}
\newcommand{\acir}[1]{\ar@{}[#1]|-{\textstyle{\circlearrowright}}}
\newcommand{\acil}[1]{\ar@{}[#1]|-{\textstyle{\circlearrowleft}}}
\newcommand{\ard}[1]{\ar@{.>}[#1]}
\newcommand{\mt}[1]{\ar@{|->}[#1]}
\newcommand{\inm}[1]{\ar@{}[#1]|-{\in}}
\newcommand{\inr}{\ar@{}[d]|-{\rotatebox[origin=c]{-90}{$\in$}}}
\newcommand{\inl}{\ar@{}[u]|-{\rotatebox[origin=c]{90}{$\in$}}}

\newcommand{\maxr}[2]{\max_{\scriptsize \begin{array}{c}{#1}\\{#2}\end{array}}}

\newcommand{\sumo}[2]{\sum_{#1=1}^{#2}}
\newcommand{\sumz}[2]{\sum_{#1=0}^{#2}}
\newcommand{\prodo}[2]{\prod_{#1=1}^{#2}}

\newcommand{\scoltwo}[2]{
\left(\begin{smallmatrix}
{#1}\\
{#2}
\end{smallmatrix}\right)}
\newcommand{\coltwo}[2]{
\begin{pmatrix}
{#1}\\
{#2}
\end{pmatrix}}

\newcommand{\smatt}[4]{
\left(\begin{smallmatrix} 
{#1}&{#2}\\
{#3}&{#4}
\end{smallmatrix}\right)
}

\newcommand{\beq}[1]{\begin{equation}\llabel{#1}}
\newcommand{\eeq}[0]{\end{equation}}
\newcommand{\bal}[0]{\begin{align*}}
\newcommand{\eal}[0]{\end{align*}}
\newcommand{\ban}[0]{\begin{align}}
\newcommand{\ean}[0]{\end{align}}

\newcommand{\llabel}[1]{\label{#1}\text{\fixme{\tiny#1}}}

\newcommand{\arxiv}[1]{\url{http://www.arxiv.org/abs/#1}}

\newcommand{\kp}[0]{\fl{d\log_2(M)}}
\newcommand{\Lad}[0]{\fc{2}{\ln 2}d}
\newcommand{\Ladm}[0]{\fc{2}{\ln 2}(d+m)}
\newcommand{\Laldm}[0]{\fc{2}{\ln 2}\ell(d+m)}
\newcommand{\Bd}[0]{B_{\Lad}^d}
\newcommand{\KuB}[0]{K_0\cup B_{\Lad}^d}
\newcommand{\Lvg}[0]{\max_{0\le k\le k'} f_{A,\KuB}(k) +  \pa{\sumz k{k'-1}f_{A,K}(k)} \pa{\Lad+1}}
\newcommand{\maxdc}[0]{\max\bc{\fc{2}{\ln 2}d,C_0}}
\newcommand{\Lp}[0]{\max_{0\le k\le p(k'+1)-1} f_A(k)\maxdc
+ \pa{\sumz k{p(k'+1)-2} f_{A,K}(k)}\pa{\Lad+2}}
\newcommand{\maxprederrorJ}[0]{\max\bc{
\fc{C_\xi c^2t}{\mu}+ac, \fc{C_\xi}{c_2^2}, \rc{c_1^{\fc 1{\al+2}}c_2^{\fc 2{\al+2}}}\pa{a\mu^{\fc{1}{\al+2}} + \fc{C_\xi t^{\rc 2}}{\mu^{\rc2-\rc{\al+2}}}}
}}
\newcommand{\maxprederrorJA}[0]{\max\bc{
\fc{C_\xi c^2t}{\mu}+\ve{A}_2c, \fc{C_\xi}{c_2^2}, \rc{c_1^{\fc 1{\al+2}}c_2^{\fc 2{\al+2}}}\pa{\ve{A}_2\mu^{\fc{1}{\al+2}} + \fc{C_\xi t^{\rc 2}}{\mu^{\rc2-\rc{\al+2}}}}
}}
\newcommand{\maxprederrorJAT}[0]{\max\bc{
\fc{C_\xi c^2T}{\mu}+\ve{(A,B)}_2c, \fc{C_\xi}{c_2^2}, \rc{c_1^{\fc 1{\al+2}}c_2^{\fc 2{\al+2}}}\pa{\ve{(A,B)}_2\mu^{\fc{1}{\al+2}} + \fc{C_\xi T^{\rc 2}}{\mu^{\rc2-\rc{\al+2}}}}
}}
\newcommand{\kdm}[0]{\fl{(d+m)\log_2(M)}}
\newcommand{\kldm}[0]{\fl{\ell(d+m)\log_2(M)}}
\newcommand{\clbd}[0]{2^{r+1}C(k'+1)^r \pa{\Ladm(C+1)+C_0+2}+4C_u\pa{\Ladm+2}}
\newcommand{\Mbd}[0]{(T+1)^{r-1} C_AC_0+T^r C_A(C_BC_u+C_\xi)}
\newcommand{\maxdmc}[0]{\max\bc{\fc{2}{\ln 2}(d+m),C_0}}

\newcommand{\ctzz}[2]{\pa{
\begin{array}{c}
#1\\
\mathbf 0\\
\hline
#2\\
\mathbf 0
\end{array}
}}

\allowdisplaybreaks[2]

\DeclareFontFamily{U}{wncy}{}
    \DeclareFontShape{U}{wncy}{m}{n}{<->wncyr10}{}
    \DeclareSymbolFont{mcy}{U}{wncy}{m}{n}
    \DeclareMathSymbol{\Sh}{\mathord}{mcy}{"58} 
\newcommand{\citep}[1]{\cite{#1}} 
\usepackage{authblk}

\usepackage{etoolbox}

\newcommand{\citealt}[2][]{\cite[#1]{#2}}
\renewcommand{\llabel}[1]{\label{#1}}

\begin{document}

\title{No-Regret Prediction in Marginally Stable Systems}

\author[1]{Udaya Ghai}
\author[2]{Holden Lee}
\author[1]{Karan Singh}
\author[1]{Cyril Zhang}
\author[1]{Yi Zhang}

\affil[1]{Princeton University, Computer Science Department \authorcr
  \{\tt ughai,karans,cyril.zhang,y.zhang\}@cs.princeton.edu}
\affil[2]{Duke University, Mathematics Department \authorcr
  \tt holee@math.duke.edu}

\date{\today\\$\;$ \\(Version 3)\footnote{V1 appeared on February 6, 2020. V2 made formatting changes. V3 improved exposition, added Appendices \ref{a:ols-regret} and \ref{s:php}, improved/corrected Theorem \ref{t:orr} and resulting polynomial factors, and added Assumption \ref{asm:obs}. V3 appears in COLT 2020.}}
\maketitle

\begin{abstract}
We consider the problem of online prediction in a marginally stable linear dynamical system subject to bounded adversarial or (non-isotropic) stochastic perturbations. This poses two challenges. Firstly, the system is in general unidentifiable, so recent and classical results on parameter recovery do not apply. Secondly, because we allow the system to be marginally stable, the state can grow polynomially with time; this causes standard regret bounds in online convex optimization to be vacuous. In spite of these challenges, we show that the online least-squares algorithm achieves sublinear regret (improvable to polylogarithmic in the stochastic setting), with polynomial dependence on the system's parameters. This requires a refined regret analysis, including a structural lemma showing the current state of the system to be a small linear combination of past states, even if the state grows polynomially. By applying our techniques to learning an autoregressive filter, we also achieve logarithmic regret in the partially observed setting under Gaussian noise, with polynomial dependence on the memory of the associated Kalman filter.
\end{abstract}

\newcommand{\KF}[0]{\textup{KF}}

\section{Introduction}
\llabel{sec:intro}
We consider the problem of sequential state prediction in a linear time-invariant dynamical system, subject to perturbations:
$$x_t = Ax_{t-1} + Bu_{t-1} + \xi_t.$$
This is a central object of study in control theory and time-series analysis, dating back to the foundational work of Kalman \citep{kalman1960new}, and has recently received considerable attention from the machine learning community. In a typical learning setting, the system parameters $A, B$ are unknown, and only the past states $x_t$ and exogenous inputs $u_t$ are observed. Sometimes, another layer of difficulty is imposed: the \emph{latent state} $x_t$ can only be observed noisily or through a low-rank transformation. These models serve as an abstraction for learning from correlated data in stateful environments, and have helped to understand empirical successes in reinforcement learning and of recurrent neural networks.

Many recent works \citep{simchowitz2018learning,sarkar2019finite} are concerned with finite-sample \emph{system identification}, in which the matrices $A, B$ can be recovered due to structure in the perturbations or inputs. These results rely on matrix concentration inequalities and careful error propagation applied to classic primitives in linear system identification, and have settled some important statistical questions about these methods. However, as in the classical control theory literature, these theorems require unrealistic assumptions of i.i.d. isotropic random perturbations and recoverability of the system, and often require the user to select the ``exploration'' inputs $u_t$ \citep{simchowitz2019learning, dean2017sample}. Furthermore, under model misspecification, the guarantees of parameter identification pipelines break down.

Another line of work seeks to obtain more flexible guarantees via the online learning framework \citep{hazan2016introduction}. Here, the goal of parameter recovery is replaced with regret minimization, the excess prediction loss compared to the best-fit system parameters in hindsight \citep{Kozdoba2019OnLineLO, hazan17learning, hazan2018spectral}. This approach gives rise to algorithms which adapt to adversarially perturbed data and model misspecification, and can be extended beyond prediction to obtain new methods for robust control \citep{arora2018towards}. However, 
these algorithms can diverge significantly from the classical parameter identification pipeline. In particular, they can be \emph{improper}, in that they may use an intentionally misspecified (e.g. overparameterized) model. Thus, these algorithms can be incompatible with parameter recovery and downstream methods.

In this work, we show that the same algorithm used for parameter identification (online least squares) has a no-regret guarantee, even in the challenging setting of prediction under marginal stability and adversarial perturbations, where recovery is impossible. More precisely, in this setting, where the state is allowed to grow polynomially with time, we show that the regret of this algorithm is sublinear, with a polynomial dependence on the system's parameters. This does not follow from the usual analysis of online least squares: the magnitude of loss functions (and associated gradient bounds) can scale polynomially with time, causing standard regret bounds to become vacuous. Instead, we conduct a refined regret analysis, including a structural \emph{volume doubling} lemma showing $x_t$ to be a small linear combination of past states.

By replacing the worst-case structural lemma with a stronger martingale analysis, we also show a polylogarithmic regret bound for least squares in the stochastic setting. Again, this analysis does not go through parameter convergence, and thus applies in the setting of unidentifiable systems and non-isotropic noise. The same techniques allow us to prove a logarithmic regret bound in the partially observed setting under Gaussian noise, with polynomial dependence on the memory of the associated Kalman filter.

\paragraph{Paper structure.}
In Section~\ref{sec:prelims}, we formally introduce the problem and the natural \emph{online least squares} algorithm. In Section~\ref{s:related}, we give an overview of related work. In Section~\ref{s:results}, we state our main results. In Section~\ref{s:sketch}, we sketch the proofs. In Section~\ref{s:ols}, we show that a structural condition on a time series gives a regret bound for online least squares. We then prove our main theorems for fully observed LDS in the adversarial setting (Section \ref{s:proof}), fully observed LDS in the stochastic setting (Section~\ref{s:proof-stoch}), and for partially observed LDS in the stochastic setting (Section~\ref{s:proof-hidden}), through establishing this structural condition. Section~\ref{s:php} provides an alternative approach.

\section{Problem setting and algorithm}
\llabel{sec:prelims}

The problem of online state prediction for a LDS falls within the framework of online least squares. We first introduce the general problem of online least squares (Section~\ref{subsec:oco}), and then specialize to the prediction problem for a fully observed (Section~\ref{subsec:prelims-lds}) or partially observed (Section~\ref{s:prelims-hidden}) LDS.
Because the observations come from a sequential process, they have extra structure that we will leverage to obtain better guarantees than for 
black-box online least squares. 
In Section~\ref{subsec:marginal}, we describe the challenges associated with marginally stable systems.
\subsection{The online least squares problem}
\llabel{subsec:oco}


In the problem of online least squares, at each time $t$ we are given $x_t\in \R^{m}$, and asked to predict $y_t\in \R^n$. We choose a matrix $A_t\in \R^{n\times m}$ and predict $\wh y_t = A_tx_t\in \R^n $. The desired output $y_t$ is then revealed and we suffer the squared loss $\ve{\hat y_t-y_t}^2$. 
A natural goal in this setting is to predict as well as if we had known the best matrix in hindsight; hence, 
the performance metric is given by the \emph{regret} with respect to $A$, defined by 
\begin{align*}
R_T(A) &=
\sumo t{T} \ve{A_tx_t-y_t}_2^2 - 
\sumo t{T} \ve{Ax_t-y_t}_2^2.
\end{align*}
Define the regret with respect to a given set $\cal K\subeq \R^{n\times m}$ by $R_T = \sup_{A\in \cal K} R_T(A)$.
In general, we would like to achieve $R_T$ that is sublinear in $T$, or equivalently, average regret $\fc{R_T}{T}$ that converges to 0.
In some cases, we can do better, and achieve regret $R_T$ that is polylogarithmic in $T$.

A natural algorithm for online least squares is to choose $A_t$ that minimizes the total squared prediction error for all the pairs $(x_s,y_s)$, $1 \le s \le t-1$ seen so far, plus a regularization term:


\begin{algorithm}[H]
\begin{algorithmic}
\INPUT Regularization parameter $\mu$. 
\For{$t=1$ to $T$}
	\State Let $A_t = \amin_{A\in \R^{n\times m}} \pa{\mu \ve{A}_F^2 + \sumo s{t-1} \ve{Ax_s - y_s}^2}$.
	\State Predict $\wh y_t := A_tx_t$ and observe cost $\ve{\wh y_t-y_t}^2$.
\EndFor
\end{algorithmic}
 \caption{Online Least Squares Regression}
 \llabel{a:orr}
\end{algorithm}

We state the standard regret bound for Algorithm~\ref{a:orr}. Note that \cite{cesa2006prediction} show the $n=1$ case, but the proof goes through in the multi-dimensional case (Appendix~\ref{a:ols-regret}).
\begin{thm}[OLS regret bound; Thm. 11.7, \citealt{cesa2006prediction}]\llabel{t:orr}
In the online least squares setting, suppose that $\ve{x_t}\le M$ for all $t$. 
Then, Algorithm~\ref{a:orr} incurs regret
\begin{align*}
R_T (A)&\le 
\mu \ve{A}_F^2 + \max_{1\le t\le T}\ve{y_t-A_tx_t}_2^2m
\ln \pa{1+\fc{TM^2}m}.
\end{align*}
\end{thm}
Thus, if there is a uniform bound on the prediction errors $\ve{y_t-A_tx_t}$, online least squares achieves logarithmic regret. This follows immediately in the usual OLS setting, where $x_t, y_t,$ are bounded and $A_t$ is restricted to a bounded set. However, in the case where they can grow with time, as in marginally stable systems, a more sophisticated analysis will be necessary to get sublinear regret.

\subsection{Prediction in fully-observed linear dynamical systems}
\llabel{subsec:prelims-lds}

A special case of online least squares is state prediction in a time-invariant linear dynamical system (LDS), defined as follows. 
Given an initial state $x_0 \in \R^d$, matrices $A \in \R^{d \times d}$ and $B\in \R^{d\times m}$, 
inputs $u_0,\ldots, u_{T-1}\in \R^m$ and 
a sequence of perturbations $\xi_1, \ldots, \xi_T \in \R^d$, the LDS produces a time series of states $x_1, \ldots, x_T \in \R^d$ according to the following dynamics:
\begin{align}\llabel{e:lds}
 x_t = A x_{t-1} + B u_{t-1} + \xi_t, \qquad 1\le t\le T.
\end{align}
This setting generalizes the \emph{linear Gaussian model} from control theory and time-series analysis, in which each $\xi_t$ is drawn i.i.d. from a Gaussian distribution. Aside from modeling disturbances, $\xi_t$ can also represent model uncertainty or misspecification.

In the prediction problem for LDS, we are asked to predict $x_{t+1}$ as a linear function of the current $x_t$ and the input $u_t$. 
We can treat this as an online least squares problem, by casting $(x_t, u_t)$ as the input at time $t$, and $x_{t+1}$ as the desired output. At each step, the learner produces $A_t$ and $B_t$ and predicts $\wh x_{t+1} = A_t x_t + B_t u_t$. Thus, we can adapt Algorithm~\ref{a:orr} to this setting with the substitution $x_t\mapsfrom \scoltwo{x_t}{u_t}$, $y_t\mapsfrom x_{t+1}$, $A\mapsfrom (A,\;B)$, $(n,m) \mapsfrom (d,d+m)$, and obtain Algorithm~\ref{a:ols-lds}. Note that we index from 0. 
Translated to this setting, the goal of regret minimization becomes that of predicting as accurately as if one had known the system's underlying matrices $A$ and $B$. 

\begin{algorithm}
\caption{Online Least Squares Regression (LDS setting)}
\begin{algorithmic}[1]
\INPUT Regularization parameter $\mu$.
\For{$t = 0, \ldots, T-1$}
  \State Estimate dynamics as 
  \[(A_t, B_t) := \argmin_{(A,B)} \pa{\mu 
  \ve{(A,B)}_F^2
  +\sum_{s = 0}^{t-1} \ve{ Ax_{s} + Bu_s - x_{s+1} }^2}.\]
  \State Predict state: $ \hat x_{t+1} := A_t x_t + B_tu_t $ and suffer loss $\ve{\hat x_{t+1} - x_{t+1}}^2$.
\EndFor
\end{algorithmic}
\llabel{a:ols-lds}
\end{algorithm}

Note that in the stochastic setting, when the covariance of the noise is lower-bounded in each direction, OLS gives a consistent estimator for $A$ and $B$; convergence rates for recovery are analyzed in \cite{sarkar2019optimal,simchowitz2018learning}. 
However, in the adversarial setting, recovery of $A$ is an ill-posed problem. 
The perturbations can be biased or rank-deficient, causing the recovery problem to be underdetermined in general, and the optimal $A$ may change as time.




\subsection{Prediction in partially-observed linear dynamical systems}
\llabel{s:prelims-hidden}

A partially-observed linear dynamical system is defined by 
\begin{align}\llabel{e:lds1}
x_t&= A x_{t-1} + B u_{t-1}+ \xi_t \\
y_t & = C h_t + \eta_t ,
\llabel{e:lds2}
\end{align}
where $u_t\in \R^m$ are inputs, $x_t\in \R^d$ are hidden states, $y_t\in \R^n$ are observations, 
$A\in \R^{d\times d}$, $B\in \R^{d\times m}$ and $C\in \R^{n\times d}$ are matrices, and $\xi_t\in \R^d$ and $\eta_t\in \R^n$ are perturbations. 
We consider the stochastic setting, so that $\xi_t$ and $\eta_t$ are independent zero-mean noise terms. 
Crucially, only the $y_t$, and not the $x_t$, are observed.

For prediction in this setting, we use Algorithm~\ref{a:ols-ar}, regressing with the previous $\ell$ observations and inputs, so  we slightly modify the definition of the regret in~\eqref{e:RTABC} to start accruing from $t=\ell+1$:
\begin{align*}
    R_T(A,B,C) &=\sum_{t=\ell+1}^{T} \ve{\wh y_t-y_t}^2 
    - \sum_{t=\ell+1}^{T} \ve{\wh y_{\KF,t}-y_t}^2,
\end{align*}
where $\wh y_{\KF,t}$ is the prediction of the \emph{steady-state Kalman filter} for the system $(A,B,C)$; see Appendix~\ref{s:proof-hidden} for a review and formal definitions. Note that we will learn the system in an \emph{improper} manner: that is, we will predict $\wh y_t$ using a general autoregressive filter, rather than the Kalman filter of some system.

\begin{algorithm}
\caption{Online Least Squares Autoregression for LDS}
\begin{algorithmic}[1]
\INPUT Regularization parameter $\mu$, rollout length $\ell$.
\For{$t = 0$ to $T-1$}
  \State Estimate the autoregressive filter: 
  $$(F_t, G_t) := \argmin_{F\in \R^{n\times \ell m}, G\in \R^{n\times \ell n}} \pa{\mu \ve{(F,G)}_F^2+\sum_{s = \ell-1}^{t-1} \ve{Fu_{s:s-\ell+1} + Gy_{s:s-\ell+1} - y_{s+1} }^2}. $$
  \State Predict state: $ \hat y_{t+1} :=F_t u_{t:t-\ell+1} + G_t y_{t:t-\ell+1}$ and suffer cost $\ve{\hat y_{t+1} - y_{t+1}}^2$.
\EndFor
\end{algorithmic}
\llabel{a:ols-ar}
\end{algorithm}

\subsection{Marginally stable systems}
\label{subsec:marginal}

In this work, we are interested in prediction in marginally stable systems. In both the fully-observed and partially-observed cases, the \emph{spectral radius} $\rho(A)$ of the system is defined to be the magnitude of the largest eigenvalue of the transition matrix $A$, as in Equations \ref{e:lds} and \ref{e:lds1}. An LDS is \emph{marginally stable} if $\rho(A) = 1$.

As opposed to \emph{strictly} stable systems (ones for which $\rho(A) < 1$), these systems model phenomena where the state does not reset itself over time, often representing physical systems which experience little or no dissipation. As discussed in Section~\ref{s:related}, their capacity to represent long-term dependences presents algorithmic and statistical challenges. An \emph{inverse spectral gap} factor $\frac{1}{1-\rho(A)}$ appears in the computational and statistical guarantees for many learning algorithms in these settings (see, e.g. \cite{hardt2016gradient}), as a finite-impulse truncation length or mixing time, rendering those results inapplicable.

Among marginally stable systems, the hardest cases are those with large Jordan blocks corresponding to large eigenvalues. Defining the \emph{Jordan matrix}
\begin{align*}
    J_{\lambda,r} := {\begin{pmatrix}\lambda &1&0&\cdots &0\\0&\lambda &1&\cdots &0\\\vdots &\vdots &\vdots &\ddots &\vdots \\0&0&0&\lambda &1\\0&0&0&0&\lambda \end{pmatrix}} \in \R^{r \times r},
\end{align*}
we see that for marginally stable systems, $\norm{ J_{\lambda,r}^t }_2 = \Omega(\lambda^{t-r+1} t^{r-1})$ can grow polynomially in $t$. These occur naturally in physical systems as discrete-time integrators of $r$-th degree ordinary differential equations. The primary challenge we overcome in this work is to show the sublinear regret of least squares, even when the state grows polynomially. As is also the case in our work, recent advances in parameter recovery of marginally stable systems \citep{simchowitz2018learning, sarkar2019optimal} exhibit exponential dependences on the largest Jordan block order $r$.

\section{Related work}
\llabel{s:related}

Linear dynamical systems have been studied in a number of disciplines, including control theory \citep{ghahramani1996parameter,kalman1963mathematical}, astronomy \citep{chiappori1992sunspot}, econometrics \citep{hendry1994modelling}, biology \citep{saunders1994evolution}, and chemical kinetics. They capture many popular models in statistics and machine learning \citep{ghahramani1996parameter}.
We first describe the results on parameter estimation and prediction in fully observed systems, and then describe results more broadly applicable to partially observed systems. Unless noted otherwise, the results hold under the assumption that the noise is i.i.d.; some results also require that it be Gaussian.

\paragraph{Fully-observed LDS.} 
\cite{dean2017sample} show that when given independent rollouts of a LDS, the least-squares estimator of the parameters is sample-efficient. Using this, they obtain sub-optimality bounds for control.
\cite{simchowitz2018learning} consider the more challenging case when only a single trajectory is given, and show that the least-squares estimator is still efficient, despite correlations across timesteps. Their results hold for marginally stable systems.
Improving over \cite{simchowitz2018learning} and \cite{faradonbeh2018finite}, \cite{sarkar2019optimal} offer bounds applicable even to explosive systems, with the restriction that explosive eigenvalues have unit geometric multiplicity. 
In order to obtain results for parameter recovery, all these results assume that the covariance of the noise is lower bounded. Because we are concerned with prediction, this requirement will not be necessary for our results.

\paragraph{Partially-observed LDS.} For a system with known parameters, the celebrated Kalman filter \citep{kalman1960new} provides an
analytic solution for the posterior distribution of the latent states and future observations given a series of observations.
When the underlying parameters are unknown, the Expectation Maximization (EM) algorithm can be used to  learn them \citep{ghahramani1996parameter}. However, due to nonconvexity of the problem, EM is only guaranteed to converge to local optima.
In the absence of process noise $\xi_t$, \cite{hardt2016gradient} show that 
gradient descent on the maximum likelihood objective converges to the global minimum; however, they require the roots of the associated characteristic polynomial to be well-separated and the system to be strictly stable.

Subspace identification methods circumvent the nonconvexity of maximum likelihood estimation.
For strictly stable systems, \cite{oymak2018non} demonstrate that the Ho-Kalman algorithm learns the Markov parameters of the underlying system at an optimal rate (in $T$), and identifies the parameters approximately up to an equivalent realization (at a 
rate of $T^{-1/4}$) under further assumptions of observability\footnote{This notion of ``observable" is not to be confused with the system being partially observable.} and controllability. \cite{simchowitz2019learning} showed that a \emph{prefiltered} variant of least-squares offers stronger guarantees that apply even to marginally stable systems and systems with adversarial noise. For strictly stable stochastic systems, \cite{sarkar2019finite} improve upon previous works to give an optimal rate for parameter identification.
We note that these works require the control inputs, and often the noise, to be Gaussian. This may not hold when the control inputs are exogenous (not under user control). In contrast, our results can handle arbitrary (bounded) control inputs.

In a notable departure from this trend,  \cite{Tsiamis2019FiniteSA} demonstrate 
optimal recovery of system parameters 
in the absence of control inputs for 
marginally stable systems. Under similar conditions, \cite{tsiamis2019sample} prove the first result that integrates 
former system identification results with a perturbation analysis for the Kalman filter to obtain error bounds on prediction. These results only apply to stochastic systems subject to persistent excitation, which our stochastic-case result does not require.

\paragraph{Prediction via improper learning.} 
For strictly stable partially observed systems without process noise, it is sufficient to learn a finite impulse response (FIR) filter on the inputs, as observed by e.g., \cite{hardt2016gradient}. \cite{tu2017non} give near-optimal sample complexity bounds for learning a FIR filter under design inputs.
\cite{hazan17learning,hazan2018spectral, arora2018towards} instead use \emph{spectral filtering} on the inputs
to achieve regret bounds that apply in the presence of adversarial dynamics and marginally stable systems, much like the present work. However, in the presence of process noise, the regret compared to the optimal filter can grow linearly. This is an inherent limitation of any FIR-based approach; see \cite{lee2019robust} for a discussion.

If the LDS is observable, then the associated Kalman filter is strictly stable~\citep{anderson2012optimal}. Using this, \cite{Kozdoba2019OnLineLO} 
show that in this case, the Kalman filter can be arbitrarily well-approximated by an autoregressive (AR) model. \cite{AnavaHMS13} give algorithms for prediction in ARMA models with adversarial noise. However, their results hold under conditions more stringent than even strict stability. \cite{Kozdoba2019OnLineLO} show that online gradient descent on the AR model gives regret bounds scaling with the size of the observations. As discussed in Section~\ref{subsec:marginal}, in the marginally stable case, this could be polynomial in the time $T$. \cite{lee2019robust} give guarantees for learning an autoregressive filter in a stricter notion of $\mathcal{H}_\infty$ norm, but require the system to be strictly stable. In concurrent work, \cite{tsiamis2020online} establish polylogarithmic regret in the marginally stable case, but their guarantee depends on the coefficients of the characteristic polynomial, which may be exponentially large (see Section~\ref{ss:exp} and~\ref{s:ch}).
 

\paragraph{Online learning.} We use tools from online learning (see \cite{hazan2016introduction,shalev2012online, cesa2006prediction} for a survey). The standard regret bounds for online least-squares scale with an upper bound on the maximum instantaneous loss, through the gradient norm or the exp-concavity factor \cite{zinkevich2003online,hazan2007logarithmic}. Our core argument shows that this quantity is sublinear in $T$ in the LDS setting. This cannot be true for online least-squares for arbitrary polynomially growing $x_t$, so black-box results cannot apply; see Appendix~\ref{s:counterexample}. We note the similarity of our approach to \cite{rakhlin2012online,rakhlin2013optimization} where the authors show that approximate knowledge of cost functions or gradients revealed one step in advance can give ``beyond worst-case'' regret bounds. 

\section{Our results}

\llabel{s:results}
\subsection{Fully-observed LDS}

For prediction in a fully-observed LDS, we show that we can achieve sublinear regret in the adversarial setting and polylogarithmic regret in the stochastic setting.


We will make the following assumptions for both theorems:
\begin{asm}\llabel{asm:1}
The linear dynamical system
$$x_t=Ax_{t-1}+Bu_{t-1}+\xi_t$$
with $u_t\in \R^m$, $x_t,\xi_t\in \R^d$, $A\in \R^{d\times d}$, $B\in \R^{d\times m}$
satisfies the following:
\begin{itemize}
\item
The initial state is bounded: $\ve{x_0}\le C_0$.
\item
The inputs are bounded: $\ve{u_t}\le C_u$.
\item
The perturbations are bounded: $\ve{\xi_t}\le C_\xi$ for $1\le t\le T$. 
\item
$\ve{(A,B)}_F\le R$, $\rho(A) \le 1$, and $A$ can be 
written in Jordan form as $A=SJS^{-1}$ where $J$ has Jordan blocks of size $\le r$ and $\ve{S}_2\ve{S^{-1}}_2\le C_A$. 
\item
$B$ satisfies $\ve{B}_2\le C_B$.
\end{itemize}
We will let $C_{\mathrm{sys}}=\max\{C_0, C_A, C_B, C_\xi\}$.
\end{asm}

We note that the bound on the perturbations $C_\xi$ is necessary. This prevents, for example, the pathological case when the system switches between two very different linear dynamical systems $x_{t} = x_{t-1}+u_{t-1}$ and $x_t = -(x_{t-1}+u_{t-1})$ and linear regret is unavoidable. We also note that $\ve{S}_2\ve{S^{-1}}_2$ (the condition number of $S$) is a standard quantity that often appears in learning guarantees.

Our main theorem in the adversarial setting is the following; see the appendix for more precise dependences on individual constants.
\begin{thm}[Sublinear regret in the adversarial setting]\llabel{t:main}
Suppose Assumption~\ref{asm:1} holds. 
Then, there is an explicit choice of regularizer $\mu$ such that Algorithm~\ref{a:orr} achieves regret
\begin{align}\llabel{e:main}
R_T(A,B) &\le T^{\fc{2r+1}{2r+2}} 
\poly(C_{\mathrm{sys}},R,(d+m), (\ln T), (\ln C_{\mathrm{sys}})) \\
&\quad +  
T^{\rc{2r+2}} \poly(C_{\mathrm{sys}},R,(d+m)^r, (\ln T)^r, (\ln C_{\mathrm{sys}})^r).
\nonumber
\end{align}
If $C_{\mathrm{sys}}=O(1)$ and $\ve{(A,B)}_F=O(d+m)$, 
$$R_T(A,B)=
O\pa{
(d+m)^{4}d^{2}  \, r^2 \cdot T^{\fc{2r+1}{2r+2}} \ln^3 T
}
$$
as $T\to \infty$. 
The dependence of $\mu$ on $T$ is $T^{\fc{2r+1}{2r+2}}$.
\end{thm}

\begin{rem*}
This is a pessimistic bound. The worst case is when the eigenvalues of the large Jordan blocks are close to 1. If $\ve{A^k}\le C'k^{r'}$, then we can replace the dependence on $r$ with $r'$, and instead suffer a $\poly(C')$ dependence.
\end{rem*}

In the case where $A$ is diagonalizable, Theorem~\ref{t:main} implies $\tilde{O}(T^{3/4})$ regret:
\begin{cor}\label{c:main}
Suppose Assumption~\ref{asm:1} holds, and further suppose $A$ is diagonalizable. 
There is an explicit choice of regularizer $\mu$ such that Algorithm~\ref{a:ols-lds} achieves regret
\begin{align}
R_T(A,B) &\le T^{3/4} 
\poly(C_{\mathrm{sys}},R,d,m, \ln T).
\end{align}
When $C_{\mathrm{sys}} = O(1)$ and $R=O(d+m)$, we have 
$R_T(A,B) = O((d+m)^{3.5} d \, T^{3/4} \ln^3 T)$ as $T\to \infty$. The dependence of $\mu$ on $T$ is $T^{3/4}$.
\end{cor}

Our main theorem in the stochastic setting is the following.
\begin{thm}[Polylogarithmic regret in stochastic setting]\llabel{t:main-stoch}
Suppose Assumption~\ref{asm:1} holds, and further that 
$\xi_t$ is a random variable satisfying $\E[\xi_t|\xi_{t-1},\ldots, \xi_1]=0$. 
Then with probability at least $1-\de$, Algorithm~\ref{a:ols-lds} with $\mu=1$ achieves regret
\begin{align}
R_T(A,B) &\le 
\poly\pa{C_{\mathrm{sys}},R,d^r,(\ln T)^r, (\ln C_{\mathrm{sys}})^r, \ln \pa{1/\delta}}.
\end{align}
\end{thm}
Note that there is no requirement that the noise be i.i.d., nor that their covariance is greater than some multiple of the identity. At the expense of a $\ln T$ factor, the theorem can be applied to subgaussian random variables, by first conditioning on the event that $\ve{\xi_t}\le C_\xi$ for all $1\le t\le T$.


\subsection{Partially-observed LDS}
\label{subsec:partially-observed-results}

Our assumptions in the partially observed setting are the following. We will assume that the noise is i.i.d. Gaussian; this is the analogue of the linear-quadratic estimation (LQE) setting where we only care about predicting the observation.

\begin{asm}\llabel{asm:obs}
The partially-observed LDS defined by~\eqref{e:lds1}--\eqref{e:lds2} satisfies the following: The perturbations are i.i.d. Gaussian with $\xi_t\sim N(0,\Si_x)$ and $\eta_t\sim N(0,\Si_y)$. Moreover, the system matrix pair $(A,C)$ is observable (i.e., $O_n = [C; CA;\ldots;CA^{n-1}]$ has rank $n$)\footnote{We use $[A_1;A_2;\ldots]$ to denote a block-column matrix.}, the pair $(A,\Si_x^{1/2})$ is controllable (i.e., $[\Si_x^{1/2},A\Si_x^{1/2},\ldots,A^{n-1}\Si_x^{1/2}]$ has rank $n$), and $\Si_y\succ 0$.
\end{asm}

\begin{asm}\llabel{asm:2}
The partially-observed LDS defined by~\eqref{e:lds1}--\eqref{e:lds2} 
satisfies the following:
\begin{itemize}
\item
The initial state has steady-state covariance: $x_0\sim \mathcal{N}(x_0^{-},\Si_{0})$ with $\|x_0^-\|\leq C_0$.
\item
The inputs are bounded: $\ve{u_t}\le C_u$.
\item
The perturbations are Gaussian: $\xi_t \sim N(0,\Si_x)$ and $\eta_t\sim N(0,\Si_y)$. 
\item
$\ve{A}_F\le R$, $\rho(A) \le 1$, and $A$ can be written in Jordan form as $A=SJS^{-1}$ where $J$ has Jordan blocks of size $\le r$ and $\ve{S}_2\ve{S^{-1}}_2\le C_A$. 
\item
$B$ and $C$ satisfy $\ve{B}_2\le C_B$ and $\ve{C}_2\le C_C$.
\end{itemize}
We will let $ C_\mathrm{sys} =\max\{C_0, C_A, C_B, C_C, \ve{\Si_0}, \ve{\Si_x},\ve{\Si_y}\}$.
\end{asm}
For simplicity, our result assumes that $x_0$ has steady-state covariance. If this is not the case, then one would need quantitative bounds on how quickly the time-varying Kalman filter converges to the steady-state Kalman filter, to bound the additional regret incurred by using a fixed filter. 

The theorem also depends on the \emph{sufficient length} $R(\eps)$ of the Kalman filter, which is roughly the length at which we can truncate the unrolled filter to incur an error of at most $\eps$; see Definition~\ref{d:suff-len} for a precise account. If the filter decays exponentially, then $R(\eps) = O\pa{\ln\pa{ 1/\eps }}$. It remains an interesting problem to handle the case where the filter is \emph{also} marginally stable.

\begin{thm}[Polylogarithmic regret for LQE]\llabel{t:kalman}
Assume Assumptions~\ref{asm:obs} and~\ref{asm:2}. 
Suppose the corresponding Kalman filter is given by $A_{\KF}$, $B_{\KF}$, and $C_{\KF}$. 
Let $R(\cdot)$ denote the sufficient length of the Kalman filter system, and choose $\ell = R\pa{\de \bra{T^r C_C C_A \pa{C_BC_u + 10\ve{\Si_x} \sqrt{d\ln \pf{T}{\de}}}}^{-1}}$. 
Let $(F_t,\;G_t) = C_{\KF}A_{\KF}^{t}B_{\KF}$ be the unrolled Kalman filter, and suppose $\sumz t{\ell-1}\ve{F_t}^2 + \ve{G_t}^2\le R^2$.

Then with probability $1-\de$, Algorithm~\ref{a:ols-ar} with $\mu=1$ achieves regret
\begin{align}
R_T(A,B,C)&\le \poly\pa{\ell, C_\mathrm{sys}, R,(d+m)^r, (\ln T)^r,  (\ln C)^r, \ln \pa{1/\de}}.
\end{align}
\end{thm}




\section{Outlines of main proofs}
\llabel{s:sketch}
In this section, we explain the key ideas behind our results, using a fully observable LDS with adversarial noise as an example. For simplicity, we sketch the proof in the case that the matrix $A$ is diagonalizable, matrix $B$ is zero, and $\|\xi_s\|_2=O(1)$, which already captures the core difficulty of the problem. In the proof sketch, we assume that all relevant parameters for the LDS are $O(1)$. 
Full proofs are in Sections~\ref{s:ols}--\ref{s:proof-hidden}.

\subsection{Regret bounds for online least squares with large inputs}

Our starting point is the regret bound for online least squares, Theorem~\ref{t:orr}, which depends on the maximum prediction error $\max_{0\le t\le T-1} \ve{y_t-A_tx_t}^2$. To obtain sublinear regret using this bound, we must show that the maximum prediction error is $o(T/\log T)$. We show in Theorem~\ref{thm:ols-regret} that this holds as long as the following structural condition (formally defined in Definition~\ref{d:q_anom_free}) on the regressor sequence holds:

\begin{df}[Anomaly-free sequences; informal]\llabel{d:anom_free}
A sequence $(x_t)_{t<T}$  is \textbf{anomaly-free} if whenever the projection of any $x_t$ onto a unit vector $w$ is large, then there must have been $\Om(|w^\top x_t|)$ indices $s<t$ for which the projection of $x_s$ to $w$ has norm at least $\Om(|w^\top x_t|)$. 
\end{df}

Intuitively, the inputs are anomaly-free if no input $x_t$ is large in a direction where we have not already seen many inputs. Note this does not hold in the general case of polynomially-bounded $x_t$; see the counterexample in Appendix~\ref{s:counterexample}. To prove Theorem~\ref{thm:ols-regret}, we first express $A_tx_t-y_t$ in terms of the preceding states and errors $\braces{ x_s, \xi_s }_{s<t}$ (Lemma~\ref{l:ls-calc}). 
Next, we show this expression is bounded in the 1-dimensional case (Lemma~\ref{l:pred-error-1}). Finally, we reduce the general $d$-dimensional case to the 1-dimensional case by diagonalizing the sample covariance matrix $\sumz s{t-1} x_sx_s^\top$. In the reduction, we project onto the eigendirections; this is why we want there to be many large inputs when projected to any direction $w$. 



Note that Theorem~\ref{thm:ols-regret} is stated more generally, allowing $\Om(|w^\top x_t|^\alpha)$ indices where $|w^\top x_s|$ is large. This allows superlinear growth in the $x_t$, and hence can be applied to dynamical systems with matrices having Jordan blocks.

\subsection{Proving LDS states are anomaly-free}
\label{subsec:structural-condition}

Our main result (Theorem~\ref{t:main}) follows by verifying that LDS states are anomaly-free. The main idea is that the evolution of the LDS ensures that $x_t$ is always approximately a linear combination of past states with small coefficients. 
More precisely, we need $x_t = \pa{\sumz s{t-1} a_sx_s} + v$ with small $\sumz s{t-1} |a_s|$ and $\ve{v}_2$ (Lemma~\ref{l:xt-l1-comb}). Once we have this, projecting onto $w$ gives $w^\top x_t=\pa{\sumz s{t-1} a_s w^\top x_s} + w^{\top} v $, showing that one of the projections $|w^\top x_s|$ is large. To obtain many indices $s$ for which this is large, we apply the same argument to the $k$-step dynamical systems defined by $A^2$, $A^3$, and so forth, keeping track of how many times an index can be overcounted. This shows the states are anomaly-free and finishes the proof of Theorem~\ref{t:main}.

We provide two approaches to decompose $x_t$ into previous states as needed.  The simpler approach provides coefficients of size $\exp(d)$, while a more involved approach provides $\poly(d)$ sized coefficients. In order to have a $\poly(d)$ dependence in the final regret bound, the latter approach is necessary.

\subsection{$\exp(d)$-sized coefficients using the Cayley-Hamilton theorem}\label{ss:exp}
To show that $x_t$ is always a linear combination of previous states with small coefficients, a first idea is to use the Cayley-Hamilton theorem. In the noiseless case, the theorem implies that the $x_t$ satisfy a recurrence $x_t = \sumo id a_i x_{t-i}$, where $a_i$ are the coefficients of the characteristic polynomial of $A$. Adding the noise back, we may get an error term $v$ of size $\sumo id |a_i|$ by inspecting how the noise propagates through this recurrence. Even though $\sumz id |a_i|$ can be as large as $2^d$, this suffices to get a bound that is sublinear in $T$ while exponential in $d$. For ease of reading, we first present this weaker result in Lemma~\ref{l:ch} in Section~\ref{s:ch}.\footnote{We note that an alternate approach to the one we present below is to find a multiple of the characteristic polynomial with small coefficients; see Appendix~\ref{a:php}.}

\subsection{$\poly(d)$-sized coefficients via volume doubling}
As an alternative, we now present a novel volume-doubling argument leading to a recurrence with only $\poly(d)$ coefficient size, which may be of independent interest. For the ease of presentation, we introduce some notation below:
 \begin{itemize}
     \item \emph{$\ell_1$-span of $S$.} $\De(S) := \set{\sum_{u\in S} a_uu}{a_u\in \R, \sum_{u\in S} |a_u|\le 1}$. For short, we denote $\De_t:=\De\left(\{x_s\}_{s=0}^t\right)$.
     \item \emph{Norm with respect to $\De(S)$.} $\ve{x}_{\De(S)}:=\min_{a_s, v}\sumz st |a_s|+\ve{v}_2~~~\text{s.t.} ~~x=\pa{\sumz st a_s x_s} + v ~\text{for}~x_s\in S$. 
     \item \emph{Set of ``outlier'' indices.} $I_t := \set{0\le s\le t}{\ve{x_s}_{\De_{s-1}}\ge \fc{2}{\ln 2}d}$.
\end{itemize} 

Now, our goal is summarized as bounding $\ve{x_t}_{\De_{t-1}}$. To this end, we first prove a upper bound the size of $|I_t|$ using a general potential-based argument, and then relate it to the norm $\ve{x_t}_{\De_{t-1}}$ by unrolling the dynamics appropriately.

\paragraph{Bounding the number of outliers.} 
We have to show that $|I_t|$ is at most polynomial in $d$ and logarithmic in $t$.
We prove a general lemma that if $\ve{x}_{\De(S)}$ is large enough ($\ve{x}_{\De(S)}\ge \fc{2}{\ln 2}d$), then adding $x$ to the $S$ increases its volume significantly: 
$\Vol\pa{\De(S\cup \{x\})}\ge 2\Vol
\pa{\De(S)}$ (Lemma~\ref{l:2vol}).
Applied to our situation, this shows that if  $\ve{x_t}_{\De_{t-1}}\ge\fc{2}{\ln 2}d$, then $\Vol(\De_t) \ge 2 \Vol(\De_{t-1})$. The total volume is bounded by $(\max_{1\le s\le t} \ve{x_s})^d$, so the number of outliers at or before time $t$ is bounded by $\log_2 \ba{(\max_{1\le s\le t} \ve{x_s})^d}$, which is polynomial in $d$ and logarithmic in $t$ (Lemma~\ref{l:ST-card}).

\paragraph{Bounding $\sumz s{t-1} |a_s|$ and $\ve{v}_2$.} We obtain an inequality showing
that $\ve{x_t}_{\De_{t-1}}$ is not much larger than $\ve{x_{t-k}}_{\De_{t-k-1}}$ for small delays $k$. To see this, note that $x_t$ is generated from $x_{t-k}$ by evolving the LDS $k$ times, keeping track of noise. The contribution from the noise here is at most $\poly(k)$. Now, it suffices to find a small $k$ such that $\ve{x_{t-k}}_{\De_{t-k-1}}$ is small, or in other words $t-k$ is not an outlier. Because the number of outliers is $O(d\log t)$, we can choose $k = O(d \log T)$, and we conclude $\ve{x_t}_{\De_{t-1}}$ is at most $\poly(d\log t)$. This argument is formalized in the proof of Lemma~\ref{l:xt-l1-comb}.



\paragraph{Controlling overcounting using prime index gaps.} One technicality is that we have to apply the same argument to the dynamical systems $x_t \approx A^p x_{t-p}$ for different values of $p$. For each value of $p$, we get an index $t-pk$ such that $w^\top x_{t-pk}$ is large. To make sure we obtain enough indices this way, in the proof of Lemma~\ref{l:many} we only take $p$ to be prime, and we use a lower bound on primorials $\prod_{\text{prime } p\le X} p$ to show that we can collect enough distinct indices.

\subsection{Stochastic cases}
Finally, we provide a brief comment on how to prove Theorems~\ref{t:main-stoch} and \ref{t:kalman}. In the fully-observed setting, we can use a martingale argument, rather than the structural result, to obtain 
$\poly\log(T)$ regret. To do this, we show that with high probability, $\max_{0\le t\le T-1} \ve{A_tv_t-v_{t+1}}$ is bounded by $\poly\log(T)$ (Lemma~\ref{l:stoch}).

Let $A_t$ be the matrix predicted by online least squares with regularization parameter $\mu$.
By Lemma~\ref{l:ls-calc}, $A_tx_t-x_{t+1} = \sumz s{t-1} \xi_{s+1} v_s^\top \Si_t^{-1}v_t - \mu A \Si_t^{-1} x_t - \xi_{t+1}$. The main term we need to bound is the first one. Let $b_s=v_s^\top \Si_t^{-1} v_t$. Pretending for a moment that $b_s$ is $\cal F_s = \si(\xi_1,\ldots, \xi_s)$-valued, by Azuma's inequality it suffices to bound the variation $\sumz s{t-1} b_s^2$. We have already shown that $v_t = \sumz s{t-1} a_sv_s$ with $\sumz s{t-1} a_s^2 \le\pa{ \sumz s{t-1} |a_s|}^2 \le O(1)$, so we can bound the variation $\sumz s{t-1} b_s^2$ in terms of $\sumz s{t-1} \ve{\Si_t^{-\rc 2} v_s}^2$. By definition $\Si_t = \mu I+ \sumz s{t-1} v_sv_s^\top $, so $\sumz s{t-1} \ve{\Si_t^{-\rc 2} v_s}^2\le d$.

However, we can't apply Azuma's inequality directly, since $b_s$ depends on $v_t$, and thus is not $\cal F_s$-valued. However, the dependence on non-$\cal F_s$-valued random variables is only through $z_t=\Si_t^{-1} v_t$, which does not depend on $s$, so we can use Azuma's inequality on an $\ep$-net of possible values for $z= \Si_t^{-1}v_t$. More precisely, note that $z_t$ has the property that $\sumz s{t-1} (v_s^\top z_t)^2$ is small, so it suffices to bound $\sumz s{t-1} \xi_{s+1} (v_s^\top z)$ for a $\ep$-net of $z$ such that $\sumz s{t-1} (v_s^\top z)^2$ is small. These $z$'s live in a $d$-dimensional space, so we only incur factors of $d$.

For the partially-observed setting, we reduce to Theorem~\ref{t:main-stoch} by lifting the state: we choose a large enough horizon $\ell$ such that truncating the unrolled Kalman filter to length $\ell$ incurs an approximation error of at most $\ep$. Then, by letting the state space be the past $\ell$ observations and inputs, the partially observed LDS is approximately described by a fully-observed LDS. This incurs an additional polynomial factor in the length $\ell$.

\section{Conclusion}
\label{s:conclusion}

We have shown that online least-squares, with a carefully chosen regularization and a refined analysis, has a sublinear regret guarantee in marginally stable linear dynamical systems, even in the most difficult cases when the state can grow polynomially. In the stochastic setting, adopting the same view of low-regret prediction as opposed to parameter recovery, we have shown logarithmic regret bounds and bypassed usual isotropic noise assumptions. We list a few plausible lines of further inquiry in appendix~\ref{a:future}.
\appendix
\printbibliography
\section{Impossibility of sublinear regret without the LDS}
\label{s:counterexample}
For sake of completeness, we include a simple one-dimensional counterexample showing the insufficiency of black-box regret bounds for online least-squares when it is only assumed that an upper bound for the regressors grows with time. Even without adversarial perturbations, this information-theoretic lower bound shows that we require a refined notion of gradual growth of regressors $x_t$, as in the structural results of Section~\ref{subsec:structural-condition}, to achieve sublinear regret.

\begin{proposition}
\label{thm:impossible-without-lds}
There is a joint distribution over $a \in [-1, 1]$ and length-$T$ sequences $x_t, y_t \in \R$, for which $y_t = ax_t$ and $\abs{x_t} \leq t$, but any online algorithm incurs at least $T^2$ expected regret.
\end{proposition}
\begin{proof}
We construct this distribution, choosing $a = \pm 1$ with equal probability. We choose $x_t = y_t = 0$ for all $1 \leq t \leq T-1$, then choose $x_T = T$, so that $y_T = a x_T$. In this example, at time $T$, all previous feedback $(x_t, y_t)$ is independent of $a$, so the best prediction at time $T$ is $\hat{y}_T=0$, which 
suffers expected least-squares loss $T^2$.
\end{proof}

\section{Anomaly-free inputs imply sublinear OLS regret}
\llabel{s:ols}


We begin by defining a structural condition on OLS inputs, whereby if any time $x_t$ is large in a direction then many previous $x_s$'s are also large in that direction.

\begin{df}[Anomaly-free sequences]\llabel{d:q_anom_free}
A sequence $(x_t)$ is $(c, c_1, c_2, \al)$-\emph{anomaly-free} if for any $t$ and any unit vector $w\in \R^d$, if $|w^\top x_t|=M> c$, there exist at least $c_1M^\al$ indices $s<t$ such that $|w^\top x_s|\ge c_2 M$. 
\end{df}

We show that when 
every $y_t$ is obtained from anomaly-free $x_t$ from a fixed linear transformation, plus a bounded perturbation, 
then OLS attains sublinear regret.
\begin{thm}\llabel{thm:ols-regret}
For constants $C_\xi, c,c_1,c_2, \al \in (0,1]$, $c_1\le \rc 2$, $c\ge \rc{c_1}$, suppose an online least-squares problem satisfies the following conditions:
\begin{enumerate}
\item There exists $A^* \in \R^{n\times m}$ such that for all $t$, $y_t=A^*x_{t}+\xi_t$ with $\ve{\xi_t}\le C_\xi$.
\item The input vectors are bounded: $\ve{x_t}\leq B$.
\item
$(x_s)_{s<t}$ is $(c, c_1, c_2, \al)$-anomaly-free. 
\end{enumerate}
Then online least squares with regularization parameter $\mu$ incurs regret 
\begin{align*}
R_t (A)&\le 
\mu \ve{A}_F^2 + O\Bigg(m^3 \maxprederrorJA^2 \\
&\quad \cdot 
\ln \pa{1+\fc{tB^2}m}\Bigg).
\end{align*}
If (1a) $\mu \ge c^{\fc{4(\al+2)}{\al+4}} c_1^{\fc{2}{\al+4}}c_2^{\fc 4{\al+4}} t^{\fc{\al+2}{\al+4}}$, (1b) $\ve{A}_2\le \fc{C_\xi ct}{\mu}$, (2) $\mu \le \fc{c_2^{\fc{4(\al+2)}{\al}}t^{\fc{\al+2}\al}}{c_1^{\fc 2\al}}$, and (3) $\mu \le \fc{t}{C_\xi^2\ve{A}^2}$, then
$
    R_t(A) \le \mu \ve{A}_F^2 + O\pf{m^3 C_\xi^2 t}{c_1^{\fc 2{\al+2}} c_2^{\fc 4{\al+2}} \mu^{1-\fc 2{\al+2}}}. 
$
If $\ve{A}_F = O(m)$, 
$C_\xi = O(1)$, and $\mu = \fc{m^{\fc{\al+2}{\al+1}}
t^{\fc{\al+2}{2(\al+1)}}}{c_1^{\rc{\al+1}} c_2^{\fc 2{\al+1}}}$ satisfies these conditions, then 
\begin{align}R_t(A) = O\pf{m^{2+\fc{1}{\al+1}}
t^{\fc{\al+2}{2(\al+1)}}\ln \pa{1+\fc{tB^2}{m}}}{c_1^{\rc{\al+1}} c_2^{\fc 2{\al+1}}}
\llabel{e:ols-simp-R}
\end{align}
\end{thm}
Note that if $t\to \infty$ while the other constants are fixed, the inequalities do hold for the choice of $\mu$. The only inequality that is not immediately clear is (1a); it follows from comparing the exponents of $t$: $\fc{\al+2}{2(\al+1)} \ge \fc{\al+2}{\al+4}$ for $\al\le 1$.

By Theorem~\ref{t:orr}, it suffices to bound $\ve{A_tx_t-y_{t}}$, which is done in Lemma~\ref{l:pred-error-d}. We start with the following calculation.
\begin{lem}\llabel{l:ls-calc}
Let $x_t\in \R^m$, $A\in \R^{n\times m}$, $\xi_{t}\in \R^n$ for $0\le s\le t-1$. Suppose $y_s=Ax_s+\xi_{s}$ for each $0\le s\le t-1$. 
Let $A_t=\amin_{A\in \R^{n\times m}} \ba{\mu \ve{A}_F^2 + \sumz s{t-1} \ve{y_t-Ax_t}_2^2}$. 
Let 
\begin{align*}
\Si_t &= \mu I_m + \sumz s{t-1} x_sx_s^\top .
\end{align*}
Then 
\begin{align}
A_t x_t - y_t &=
\sumz s{t-1} \xi_{s}x_s^\top \Si_t^{-1} x_t
- \mu A \Si_t^{-1} x_t - \xi_{t}.
\llabel{e:pred-error}
\end{align}
\end{lem}
\begin{proof}
We calculate
\begin{align*}
A_t &= \amin_A \ba{\mu \ve{A}_F^2 + \sumz s{t-1} \ve{y_t-Ax_t}_2^2}\\
&=\sumz s{t-1} y_sx_s^\top  \Si_t^{-1}= \sumz s{t-1}\pa{ Ax_sx_s^\top \Si_t^{-1} + \xi_{s} x_s^\top \Si_t^{-1}}\\
A &= A \pa{\sumz s{t-1}x_sx_s^\top + \mu I_m} \Si_t^{-1} = \sumz s{t-1} Ax_sx_s^\top \Si_t^{-1} + \mu A \Si_t^{-1}\\
A_t-A &= \sumz s{t-1} \xi_{s} v_s^\top \Si_t^{-1} - \mu A \Si_t^{-1}\\
A_t x_t - y_t &= (A_t-A)x_t - \xi_{t}\\
&= \sumz s{t-1} \xi_{s}x_s^\top \Si_t^{-1} x_t
- \mu A \Si_t^{-1} x_t - \xi_{t}.
\qedhere
\end{align*}
\end{proof}

We now bound the size of the residual when condition~3 from Theorem~\ref{thm:ols-regret} holds.  We focus on the one-dimensional case in Lemma~\ref{l:pred-error-1} where the scale of the unknown parameter is $a$.  Afterwards, we bound the residuals in the general case by diagonalization and reduction to the one-dimensional result in Lemma~\ref{l:pred-error-d}.

\begin{lem}\llabel{l:pred-error-1}
Let $c,c_1,c_2$ be constants, $c_1\le \rc 2$, $c\ge \rc{c_1}$.  Let $0<\al\le 1$. 
Suppose that $(x_s)_{s<t}$ is $(c, c_1, c_2, \al)$-anomaly-free and suppose $|\xi_s|\le C_\xi$ for each $s$, then for $a>0$, 
\begin{align*}
&\sumz s{t-1} \ab{\xi_{s} x_s \pa{\mu + \sumz r{t-1} x_{r}^2}^{-1}x_t}
+a
\ab{\mu\pa{\mu + \sumz s{t-1} x_{s}^2}^{-1}x_t}\\
&\le 
O\pa{\maxprederrorJ
}.
\end{align*}
\end{lem}
Recall, from Theorem~\ref{t:orr}, the univariate regret then scales as the square of this quantity plus a regularization term scaling linearly in $\mu$. For the case $\al=1$, to balance the terms, we set $\pf{t^{\rc 2}}{\mu^{\rc 6}}^2\sim \mu$, so $\mu = t^{\fc 34}$ with regret $\wt O(t^{\fc 34})$.
\begin{proof}
Let $M=\max_{0\le s\le t-1} |x_s|$. 

First suppose $M\le c$. Then 
\begin{align*}
\ab{\sumz s{t} \xi_{s} x_s \pa{\mu + \sumz r{t-1} x_{r}^2}^{-1}x_t}\le 
\fc{C_\xi c^2t}{\mu}
\end{align*}
and the second term satisfies
\begin{align*}
\ab{\mu \pa{\mu + \sumz s{t-1} x_{s}^2}^{-1}x_t}\le 
c.
\end{align*}

Now suppose $M> c$.
Then there exists a subsequence $x_{i_1},\ldots, x_{i_n}$ where $n = \ce{c_1M^\al}>1$ and $i_j<t$, $|x_{i_j}|\ge c_2M$ for each $j$. Let $S=\{i_1,\ldots, i_n\}$. We have
\begin{align*}
\sumo k{n} |x_{i_k}| &\le \ce{c_1 M^\al}M=O(c_1M^{\al+1})&
\sumo k{n} x_{i_k}^2 &\ge \ce{c_1 M^\al}(c_2M)^2 = \Omega(c_1c_2^2M^{\al+2}).
\end{align*}
Thus
\begin{align*}
&\sumz s{t-1} \ab{\xi_{s+1} x_s \pa{\mu + \sumz r{t-1} x_{r}^2}^{-1}}
+
a\ab{\mu\pa{\mu + \sumz s{t-1} x_{s}^2}^{-1}}
\le \fc{{{a\mu+C_\xi\sumo s{t-1} |x_s|}}}{\mu+\sumo s{t-1}x_s^2}\\
&\precsim
\fc{a\mu+C_\xi(c_1 M^{\al+1}+\sum_{i\in [t-1]\bs S} |x_{i}|)}{\mu+c_1c_2^2 M^{\al+2}+ \sum_{i\in [t-1]\bs S} x_{i}^2}\\
&\le
\max_{0\le x'_i\le M\text{ for }1\le i\le t-1-n} 
\fc{a\mu+C_\xi(c_1 M^{\al+1}+\sumo i{t-1-n} x_i')}{\mu+c_1c_2^2 M^{\al+2}+\sumo i{t-1-n} x_i^{\prime2}}\\
&\le \max_{0\le x\le M}
\fc{a\mu + C_\xi(c_1 M^{\al+1}+(t-1-n) x)}{\mu+c_1c_2^2 M^{\al+2}+(t-1-n) x^2}
\end{align*}
where the last inequality follows by Jensen's inequality since $x\mapsto x^2$ is convex (holding $\sumo i{t-1-n} x_i'$ constant, the smallest value of $\sumo i{t-1-n}x_i^{\prime 2}$ is attained when they are equal). Continuing,
\begin{align*}
&\le \max_{0\le x\le M} \pa{\fc{C_\xi c_1M^{\al+1}}{\mu + c_1c_2^2M^{\al+2}} + \fc{a\mu+C_\xi(t-1-n)x}{\mu+c_1c_2^2 M^{\al+2}+(t-1-n) x^2}
}\\
&\le \max_{0\le x\le M} \pa{\fc{C_\xi c_1M^{\al+1}}{\mu + c_1c_2^2M^{\al+2}} + \fc{a\mu+C_\xi tx}{\mu+c_1c_2^2 M^{\al+2}+t x^2}
}
\end{align*}
because $x\mapsto \fc{ax}{b+cx}$ is increasing for $a,b,c>0$ and $x\ge 0$.
Thus
\begin{align*}
&\sumz s{t-1} \ab{\xi_{s} x_s \pa{\mu + \sumz r{t-1} x_{r}^2}^{-1}x_t}
+
a\ab{\mu\pa{\mu + \sumz s{t-1} x_{s}^2}^{-1}x_t}\\
&\precsim
M\cdot 
\max_{0\le x\le M} \pa{\fc{C_\xi c_1M^{\al+1}}{\mu + c_1c_2^2M^{\al+2}} + \fc{a\mu+C_\xi tx}{\mu+c_1c_2^2 M^{\al+2}+tx^2}
}\\
&\le 
\max_{0\le x\le M} \pa{\fc{C_\xi}{c_2^2} + \fc{aM\mu+ C_\xi Mtx}{\mu+c_1c_2^2 M^{\al+2}+t x^2}
}
\end{align*}
Applying the weighted weighted arithmetic-geometric mean inequality, the  second term is bounded by
\begin{align*}
&aM\mu \prc{\mu}^{\fc{\al+1}{\al+2}} \prc{c_1c_2^2M^{\al+2}}^{\rc{\al+2}}
+C_\xi Mtx \prc{\mu}^{\rc 2-\rc{\al+2}} \prc{c_1c_2^2M^{\al+2}}^{\rc{\al+2}}\prc{tx^2}^{\rc 2} \\
&
\le \rc{c_1^{\fc 1{\al+2}}c_2^{\fc 2{\al+2}}}\pa{a\mu^{\fc{1}{\al+2}} + \fc{C_\xi t^{\rc 2}}{\mu^{\rc2-\rc{\al+2}}}}.
\qedhere
\end{align*}
\end{proof}

Now we are ready to bound~\eqref{e:pred-error} in the multidimensional case.
\begin{lem}\llabel{l:pred-error-d}
Suppose the conditions specified in Theorem~\ref{thm:ols-regret} are satisfied, then
for each $0\le t\le T-1$,
\begin{align*}
\ve{A_tx_t - y_t}&\le O\pa{m \maxprederrorJA }.
\end{align*}
\end{lem}
\begin{proof}
We bound each of the terms in~\eqref{e:pred-error}.

Because $\sumz s{t-1}x_sx_s^\top$ is symmetric, we can find an orthogonal matrix  $U$ such that $U\pa{\sumz s{t-1} x_sx_s^\top}U^\top$ is diagonal. Let $z_s = Ux_s$. Then
\begin{align}
\nonumber
\sumz s{t-1} \xi_{s} x_s^\top \pa{\mu I_m + \sumz r{t-1} x_rx_r^\top}^{-1}x_t
&= 
\sumz s{t-1} \xi_s x_s^\top U^\top \pa{\mu I_m + \sumz r{t-1} Ux_rx_r^\top U^\top }^{-1}Ux_t\\
\nonumber
&=\sumz s{t-1} \xi_s z_s^\top \pa{\mu I_m + \sumz r{t-1} z_rz_r^\top}^{-1}z_t\\
&=\sumo im\sumz s{t-1} \xi_s z_{si} \pa{\mu  + \sumz r{t-1} z_{ri}^2}^{-1} z_{ti}\llabel{e:ols-sum1}
\end{align}
where we can ``decouple" the coordinates because $\mu I + \sumo s{t-1} z_sz_s^\top$ is diagonal.
Similarly, 
\begin{align}
\nonumber
\mu A\pa{\mu I_m + \sumz s{t-1} x_sx_s^\top}^{-1}x_t &=\mu AU^\top \pa{\mu I + \sumz s{t-1} z_sz_s^\top}^{-1}z_t\\
&=\sumo im (AU^\top)_{\cdot i} \mu\pa{\mu I_m + \sumz s{t-1} z_{si}^2}^{-1}z_{ti}
\llabel{e:ols-sum2}
\\
\nonumber
\ve{\mu A\pa{\mu I_m + \sumz s{t-1} x_sx_s^\top}^{-1}x_t }
&\le \max_{1\le i\le m} \ve{(AU^\top)_{\cdot i}} \sumo im \ab{\mu \pa{\mu I + \sumz s{t-1} z_{si}^2}^{-1}z_{ti}}\\
\nonumber
&\le \ve{A}_2\sumo im \ab{\mu \pa{\mu I_m + \sumz s{t-1} z_{si}^2}^{-1}z_{ti}}
\end{align}
From the hypothesis of the lemma and the fact that $z_{ti} = U_{i\cdot}x_t$, we conclude that $z_{ti}$ satisfies the conditions of Lemma~\ref{l:pred-error-1}.
Apply Lemma~\ref{l:pred-error-1} to the $z_{ti}$ to obtain bounds for 
\begin{align*}\sumz s{t-1} \ve{\xi_s}\ab{z_{si} \pa{\mu I + \sumz r{t-1} z_{ri}^2}^{-1} z_{ti}} + \ve{A}_2 \ab{\mu\pa{\mu I + \sumz s{t-1} z_{si}^2}^{-1}z_{ti}}\end{align*}
for each $i$. Summing over over $i$ gives a factor of $m$, and gives a bound for $\eqref{e:ols-sum1}+\eqref{e:ols-sum2}$. 
Finally, the term $-\xi_{t}$ contributes at most $C_\xi$, which can be absorbed into the bound.
\end{proof}

\begin{proof}[Proof of Theorem~\ref{thm:ols-regret}]
The theorem follows from plugging in the bound on $\ve{A_tx_t-y_t}$ in Lemma~\ref{l:pred-error-d} into the regret bound of Theorem~\ref{t:orr}.

For the second statement, 
we check that (1a) implies 
$\fc{C_\xi^2 c^2t}{\mu} \le 
\rc{c_1^{\rc{\al+2}}c_2^{\fc2{\al+2}}}\cdot \fc{C_\xi t^{\rc 2}}{\mu^{\rc 2-\rc{\al+2}}}$, (1b) implies that $\ve{A}_2c \le \fc{C_\xi c^2t}{\mu}$, (2) implies that $\fc{C_\xi}{c_2^2} \le \rc{c_1^{\rc{\al+2}}c_2^{\fc2{\al+2}}}\cdot \fc{C_\xi t^{\rc 2}}{\mu^{\rc 2-\rc{\al+2}}}$, and (3) implies that $\ve{A}_2 \mu^{\rc{\al+2}}\le \fc{C_\xi t^{\rc 2}}{\mu^{\rc 2-\rc{\al+2}}}$. Then the maximum term is $O\pa{\rc{c_1^{\rc{\al+2}}c_2^{\fc2{\al+2}}}\cdot \fc{C_\xi t^{\rc 2}}{\mu^{\rc 2-\rc{\al+2}}}}$.

For the third statement, we note $\mu$ was chosen to equate the two terms.
\end{proof}

\section{Proof of Theorem~\ref{t:main} (fully observed system, adversarial noise)}
\llabel{s:proof}

We complete the details for the proof outline from Section~\ref{s:sketch}. We start with the following simple observation: if $x_T$ is a linear combination of previous $x_t$'s with small coefficients, and $x_T$ has large projection in some direction, then at least one of the $x_t$'s also has large projection in that direction.

\begin{lemma}\llabel{c:exists-large}
Suppose there exist $a_t\in \R$ and $v\in \R^d$ such that 
\begin{align}
x_T &= \sumz t{T-1} a_t  x_t +  v
\end{align}
and $\sumz t{T-1}|a_t|\le L_a$, $\ve{v}_2\le L_v$. 

Then for any unit vector $w\in \R^d$, there exists $0\le t\le T-1$ such that 
\begin{align}
\ab{w^\top x_t} &\ge \fc{|w^\top x_T| - L_v}{L_a}.
\end{align}
\end{lemma}
\begin{proof}
We have  
\begin{align}
w^\top x_T &= \sumz t{T-1} a_t w^\top x_t + w^\top v~.
\end{align}
Applying H\"older's inequality, there exists $0\le t\le T-1$ such that 
\begin{align}
|w^\top x_t| &\ge \fc{|w^\top x_T - w^\top v|}{\sumz t{T-1} |a_t|} \ge \fc{|w^\top x_T|-|w^\top v|}{L_a}.
\end{align}
We then note that $|w^\top v| \leq \ve{v}$ as $w$ is a unit vector, and the
result follows.
\end{proof}
\subsection{State is a small $\ell_1$ combination of previous states ($\exp(d)$ version)}
\label{s:ch}

As a warm-up, we first give a simpler proof that obtains a bound on $\sumo t{T-1} |a_t|$ that is exponential in $d$. This subsection is for exposition purposes only and may be skipped. In the next subsection we obtain a $\poly(d)$ bound. (We note that an alternate way of obtaining a $\poly(d)$ bound is by using Lemma~\ref{l:ch} with multiple of the characteristic polynomial with small coefficients; see Appendix~\ref{a:php}.)

\begin{lem}[Large past $x$'s via Cayley-Hamilton] 
\label{l:ch}
Given Assumptions~\ref{asm:1} with $B=O$, suppose that $A$ is diagonalizable as $A=VDV^{-1}$ where $\ve{V}\ve{V^{-1}}\le C_A$. 
Then 
for any unit vector $w\in\R^d$, 
there exist $\min\Big\{\ff{|w^\top x_t|}{d^2 2^{d+1}C_AC_\xi},\allowbreak  \ff{t}{d^2}\Big\}$ values of $s$, $0\le s\le t-1$, such that 
\begin{align*}
|w^\top x_s| &\ge \fc{|w^\top x_t|}{2^{d+1}}.
\end{align*}
\end{lem}
Note that in fact the $d$ in the bound can be replaced with the degree of the minimal polynomial.

\begin{proof}
By the Cayley-Hamilton Theorem, if $p(x)=\sumz i{d} a_i^{(1)} x^i$ is the characteristic polynomial of $A$, then $\sumz id a_i^{(1)} A^i=O$. We can bound the size of the coefficients as follows. Let $r_1,\ldots, r_d$ be the roots of $p(x)=0$. Then $p(x) = \prodo id (x-r_i)$. Because $\rh(A)\le 1$, every root satisfies $|r_i|\le 1$. By the triangle inequality, the coefficients of $p(x)$ are at most the corresponding coefficients of $q(x)= \prodo id (x+|r_i|)$ in absolute value. The sum of coefficients of $q(x)$ is $q(1)\le 2^d$, so $\sumz id |a_i^{(1)}|\le 2^d$.


We now proceed to bound the error term due to the noise. 
By unfolding the recurrence we obtain
\begin{align}
\nonumber
    x_t &= A^{d} x_{t-d} + \sum_{\tau=1}^d A^{d-\tau}\xi_{t-d+\tau}\\
    \nonumber
        &= \sum_{i=0}^{d-1}-a_i^{(1)} A^i x_{t-d} + \sum_{\tau=1}^d A^{d-\tau}\xi_{t-d+\tau}\\
        \nonumber
        &= \sum_{i=0}^{d-1}-a_i^{(1)} \left(x_{t-d+i} - \sum_{\tau=1}^{i} A^{i-\tau} \xi_{t-d+\tau} \right) + \sum_{\tau=1}^d A^{d-\tau}\xi_{t-d+\tau}\\
        \label{e:ch-recur}
        &= \sum_{i=0}^{d-1}-a_i^{(1)} x_{t-d+i} \ub{+ \sum_{i=0}^{d-1}a_i^{(1)}\sum_{\tau=1}^{i} A^{i-\tau} \xi_{t-d+\tau} + \sum_{\tau=1}^d A^{d-\tau}\xi_{t-d+\tau}}{=:v^{(1)}}
\end{align}
Note that $\ve{A^i\xi}\le \ve{VD^iV^{-1}}\ve{\xi}\le C_A \ve{\xi} \leq C_AC_\xi$. 
We write $x_t = -\sum_{i=1}^{d-1}a_i^{(1)} x_{t-d+i} + v^{(1)}$ where
\begin{align*}
    \|v^{(1)}\|
    &= \left\| \sum_{i=0}^{d-1}a_i^{(1)}\sum_{\tau=1}^{i} A^{i-\tau} \xi_{t-d+\tau} + \sum_{\tau=1}^d A^{d-\tau}\xi_{t-d+\tau} \right\|
    \\ 
    &\leq \sum_{i=0}^{d-1}|a_i^{(1)}| \sum_{\tau=1}^{i} 
    C_A
    \|\xi_{t-d+\tau}\|
    + \sum_{\tau=1}^{d} 
    C_A
    \|\xi_{t-d+\tau}\|
    \\
    &\leq 
    2^d (d-1)C_AC_\xi + dC_AC_\xi\le d2^dC_AC_\xi
\end{align*}
If $w\in \R^d$ is a unit vector such that $|w^\top x_t|\ge d2^{d+1} C_AC_\xi$, we have $|w^\top x_t| \ge 2\ve{v^{(1)}}$. Noting that 
$w^\top x_t = \sumz i{d-1} - a_i^{(1)} x_{t-d+i} + v^{(1)}$, by Lemma~\ref{c:exists-large} there exists an index $0\le i\le d-1$ such that 
\begin{align*}
    |w^\top x_{t-d+i}| \geq \frac{|w^\top x_{t}| - \|v^{(1)}\|
    }{\sum_{i=0}^{d-1} |a_i^{(1)}|} \ge 
    \fc{|w^\top x_t|}{2\sumz i{d-1} |a_i^{(1)}|}
    \ge \fc{|w^\top x_t|}{2^{d+1}}
\end{align*}
In order to obtain many large past $x$'s, we apply the same argument on sequences $x_t, x_{t-k}, \ldots, x_{t-kd}$ by considering the recurrence
\begin{align}
    x_t &= A' x_{t-k} + \xi_t^{(k)}\\
    A'&=A^k\\
    \xi_t^{(k)} &= \sumo jk A^{k-j} \xi_{t-k+j}.
\end{align}
Let $p_k(x)=\sumz id a_{i}^{(k)} x^i$ be the characteristic polynomial of $A^k$. Defining $v^{(k)}$ similarly to $v^{(1)}$, we have
\begin{align}
    \|v^{(k)}\|
    &= \left\| \sum_{i=0}^{d-1}a_i^{(k)}\sum_{\tau=1}^{i} A^{\prime i-\tau} \xi_{t-dk+\tau k}^{(k)} + \sum_{\tau=1}^d A^{\prime  d-\tau}\xi_{t-dk+\tau k}^{(k)} \right\|
    \\ 
    &\leq 
    k2^d (d-1)C_AC_\xi + kdC_AC_\xi\le kd2^dC_AC_\xi
\end{align}
so that we obtain $x_t=-\sum_{i=1}^{d-1}a_i^{(k)} x_{t-k(d-i)}+v^{(k)}$ with $\sumz id |a_i^{(k)}|\leq 2^d$ and $\|v^{(k)}\|
\leq kd2^d C_AC_\xi$. We then pick $k=1, 2, \ldots, \min\bc{\left\lfloor\frac{|w^\top x_t|}{2d2^d C_AC_\xi}\right\rfloor, \ff td}$. For each choice of $k$, we know by design that $|w^\top x_t|\geq 2\|v^{(k)}\|_2$, and therefore there must exist an $x$ in the sequence $x_{t-k}, \ldots, x_{t-kd}$ such that $|w^\top x|\ge \frac{|w^\top x_t|}{2^{d+1}}$. In this way, we are able to collect in total $L=\min\bc{\left\lfloor\frac{|w^\top x_t|}{2d2^d C_AC_\xi}\right\rfloor, \ff td}$ many such $x$'s. 
To finish the argument, we note that out of the $L$ collected $x$'s, there are at least $\lfloor \frac{L}{d} \rfloor$ distinct ones, since one $x$ can appear in at most $d$ different sequences.
\end{proof}

Plugging in the conclusion of Lemma~\ref{l:ch} into Theorem~\ref{thm:ols-regret}, we can obtain the following theorem.
\begin{thm}
Assume Assumptions~\ref{asm:1}, and furthermore suppose $A$ is diagonalizable and $u_t=0$ for each $t$. Then online least squares (Algorithm~\ref{a:orr}) with $\mu = T^{\fc 34}$ achieves regret
\begin{align*}
    R_T(A,B) &\le T^{\fc 34} \exp(O(d)) \poly(C,R,\ln T).
\end{align*}
\end{thm}
We omit the details, as we will prove the stronger Theorem~\ref{t:main}.
\subsection{State is a small $\ell_1$ combination of previous states ($\poly(d)$ version)}

We will use the following notation to keep track of the growth of $A^k$.

\begin{df}
Given a matrix $A\in \R^{d\times d}$, define $f_A(k):= \ve{A^k}$. 

Given a set $K\subeq \R^d$, define 
$f_{A,K}(k):= \max_{\xi\in K}\ve{A^k\xi}$. 
\end{df}
For a rougher bound, note we can bound $f_{A,K}(k) \le f_A(k)\max_{\xi\in K}\ve{\xi}$, but keeping $f_{A,K}$ separate allows us to separately bound the dependence on the size of the starting state and on the perturbations.

The first step of the proof is to show the following. Denote the $L^2$ ball of radius $r$ in $d$ dimensions by $B_r^d = \set{\ve{x}_2\le r}{x\in \R^d}$. 


\begin{lem}\llabel{l:xt-l1-comb}
Suppose that 
$x_0\in K_0\sub \R^d$ and 
$x_t=Ax_{t-1}+\xi_t$ with $A\in \R^{d\times d}$ and $\xi_t\in K$ for $1\le t\le T$. 
Suppose $M\ge 1$ and $\ve{x_t}\le M$ for $0\le t\le T$. 
Then there exist $a_t\in \R$ and $v\in \R^d$ such that 
\begin{align}
x_T &= \sumz t{T-1} a_t x_t + v
\end{align}
and letting $k'=\kp$,
\begin{align}
\sumz t{T-1} |a_t| &\le \fc{2}{\ln 2} d\\
\llabel{e:Lvg}
\ve{v}_2 &\le \Lvg.
\end{align}
\end{lem}
We emphasize that we will only need the existence of $a_t$, $x_t$, and $v$; the algorithm will not need to compute the linear decomposition. We mention that the related notion of volumetric spanners based on the $L^2$ norm that has been applied to online convex optimization~\citep{hazan2016volumetric}.

Note that we can bound $f(k)$ in terms of the size of the Jordan blocks (Lemma~\ref{l:J-power}). To prove Lemma~\ref{l:xt-l1-comb}, we will use two lemmas (Lemma~\ref{l:2vol} and Lemma~\ref{l:ST-card}).


\subsubsection{Bounding $I_t$}

Define the convex set spanned by a set $S\sub \R^d$ by 
\begin{align}
\llabel{e:def-De}
\De(S) :&= \set{\sum_{u\in S} a_uu}{a_u\in \R, \sum_{u\in S} |a_u|\le 1, a_u=0\text{ for all but finitely many }u}. 
\end{align}
Suppose that $S$ is compact and spans $\R^d$. The centrally symmetric convex set $\De(S)$ defines the norm (called the Minkowski functional of $\De(S)$) 

\begin{align}
\ve{x}_{\De(S)} &= 
\inf_{\scriptsize \begin{array}{c}{x=\sum_{u\in S} a_u u}\\{a_u=0\text{ f.a.b.f.m. }u}\end{array}} \sum_{u\in S} |a_u|,
\end{align}
where f.a.b.f.m. abbreviates ``for all but finitely many."
First we note that we can in fact replace the inf with the min over linear combinations of $d+1$ terms.
\begin{pr}
Suppose $S\sub \R^d$ is compact and spans $\R^d$.
The following hold:
\begin{align*}\De(S) &= \set{\sumo i{d+1} a_i u_i}{a_i\in \R, u_i\in S, \sumo i{d+1} |a_i|\le 1}\\
\ve{x}_{\De(S)} &= 
\min_{x=\sumo i{d+1} a_i u_i, u_i\in S}\sumo i{d+1} |a_i|.
\end{align*}
\end{pr}
\begin{proof}
If $v\in \De(S)$ has a representation $v=\sum_{u\in S} a_u u$ in the form of~\eqref{e:def-De}, then by Carath\'eodory's Theorem on $\fc{v}{\sum_{u\in S}|a_u|}$, it can be written also as $v=\sumo i{d+1} b_i u_i$ for $\sumo i{d+1}|b_i|\le \sum_{u\in S}|a_u|$, $u_i\in S$. 

Next, we claim that $\De(S)$ is closed. To see this, note that $\De(S)$ is the image of the compact set $S^{d+1}\times \set{(a_i)_{i=1}^{d+1}\in \R^{d+1}}{\sumo i{d+1}|a_i|\le 1}$ under the continuous map
\begin{align*}
    ((a_i)_{i=1}^{d+1}, (u_i)_{i=1}^{d+1})
    &\mapsto 
    \sumo i{d+1}a_iu_i.
\end{align*}
The image of a compact set is compact, and a compact set in $\R^d$ is closed and bounded, by the Heine-Borel Theorem.

Finally, this implies $\ve{x}_{\De(S)} = \inf_{x=\sum_{u\in S} a_u u}\sum_{u\in S} |a_u| = 
\min_{x=\sumo i{d+1} a_i u_i, u_i\in S}\sumo i{d+1} |a_i|.$
\end{proof}
Define 
\begin{align}\De'(S) = \De(S\cup \set{v}{\ve{v}_2\le 1})
 = \set{\pa{\sum_{u\in S}a_uu}+v}{a_u\in \R, {a_u=0\text{ f.a.b.f.m. }u} ,v\in \R^d, \sum_{u\in S}|a_u| + \ve{v}_2\le 1}\end{align}
(note that the $L^2$ norm is used with $v$). Then 
\begin{align}
\ve{x}_{\De'(S)} &= \min_{y=\pa{\sum_{u\in S}a_uu}+v}\pa{ \sum_{u\in S}|a_u|+\ve{v}_2}.
\end{align}
Now define
\begin{align}
\De_t &= \De'(\set{x_s}{0\le s\le t})
\end{align}
for $-1\le t\le T$.
We need to keep track of the number of times when $x_t$ is not a small linear combination of previous $x_s$'s, so let
\begin{align}
I_t &= \set{0\le s\le t}{\ve{x_s}_{\De_{s-1}}\ge \Lad}.
\end{align}
We first show that if $s\in I_t$, then there is a large increase in volume between $\De_{s-1}$ and $\De_s$. Then, we use a bound on the volume of $\De_T$ to bound $|I_T|$.

\begin{lem}\llabel{l:2vol}
Let $w$ be such that $\ve{w}_{\De(S)}\ge 2Cd$. Then 
\begin{align}
\Vol(\De(S\cup \{w\})) &\ge (1+2e^{-1/C}) \Vol (\De(S)).
\end{align}
\end{lem}
\begin{proof}
Consider the set $\De(S)$ slightly shrunk, and translated along the direction of $w$:
\begin{align}
\pm \rc{Cd}w +\pa{1-\rc{Cd}} \De(S) \subeq \De(S\cup \{w\}).
\end{align}
Now for $u\in \De(S)$, $\ve{\pm \rc{C d} w + \pa{1-\rc{Cd}} u}_{\De(S)}
\ge \rc{C d}\ve{w}_{\De(S)} - \ve{u}_{\De(S)}>
\rc{Cd}(2Cd)-1\ge 1
$. Hence, 
$\pm \rc{Cd}w +\pa{1-\rc{Cd}} \De(S)$ does not intersect $\De(S)$. This means that
\begin{align}
\Vol(\De(S\cup \{w\})) &\ge \Vol(\De(S)) + \Vol\pa{\pm \rc{C d}w +\pa{1-\rc{Cd}} \De(S) }\\
&\ge \Vol(\De(S)) + 2 \pa{1-\rc{C d}}^{d} \Vol(\De(S))\\
&\ge (1+2e^{-1/C}) \Vol(\De(S)).
\end{align}
\end{proof}

\begin{lem}\llabel{l:ST-card}
Let $M=\max_{0\le t\le T} \ve{x_t}$. Then
$|I_{T}|\le d\log_2 M$.
\end{lem}
\begin{proof}
If $|I_t|>|I_{t-1}|$, then by definition $\ve{x_t}_{\De_{t-1}}\ge \fc{2}{\ln 2}d$. By Lemma~\ref{l:2vol}, $\Vol(\De_t)\ge 2\Vol(\De_{t-1})$. Hence
\begin{align}
2^{|I_T|} \Vol(B_1^d)= 2^{|I_T|} \Vol(\De_{-1}) \le \Vol(\De_T)\le M^d\Vol(B_1^d).
\end{align}
Taking logarithms gives $|I_T|\le d\log_2 M$.
\end{proof}

\subsubsection{Bounding $\ve{x_t}_{\De_{t-1}}$}
\begin{proof}[Proof of Lemma~\ref{l:xt-l1-comb}]
First we show that if $x_t$ is a linear combination of previous $x_s$'s with small coefficients, then so is $x_{t+k}$ for small $k$. 
Suppose that 
\begin{align}
x_t &= \sumz s{t-1} a_sx_s + v.
\end{align}
We claim by induction that for any $k\ge 0$,
\begin{align}
x_{t+k} &= \sumz s{t-1} a_s x_{s+k} + A^k v + \sumo jk A^{k-j} \pa{\sumz s{t-1} -a_s\xi_{s+j} + \xi_{t+j}}.
\llabel{e:xt+k}
\end{align}
Indeed, if this holds for $k$, then
\begin{align}
x_{t+k+1} &= Ax_{t+k}+\xi_{t+k+1}\\
&= A\pa{\sumz s{t-1} a_s x_{s+k} + A^k v + \sumo jk A^{k-j} \pa{\sumz s{t-1} -a_s\xi_{s+j} + \xi_{t+j}}} + \xi_{t+k+1}\\
&=\pa{\sumz s{t-1} a_s Ax_{s+k} + A^{k+1} v + \sumo jk A^{k+1-j} \pa{\sumz s{t-1} -a_s\xi_{s+j} + \xi_{t+j}}} + \xi_{t+k+1}\\
&=\pa{\sumz s{t-1} a_s (x_{s+k+1}-\xi_{s+k+1}) + A^{k+1} v + \sumo jk A^{k+1-j} \pa{\sumz s{t-1} -a_s\xi_{s+j} + \xi_{t+j}}} + \xi_{t+k+1}\\
&= \sumz s{t-1} a_s x_{s+k+1} + A^{k+1} v + \sumo j{k+1} A^{k+1-j} \pa{\sumz s{t-1} -a_s\xi_{s+j+1} + \xi_{t+j+1}}.
\end{align}
Now we consider the sizes of the coefficients in~\eqref{e:xt+k}. If $\sumz s{t-1}|a_s|=L_a$ and $\ve{v}_2=L_v$, then~\eqref{e:xt+k} expresses $x_{t+k+1}$ as a linear combination with 
\begin{align}\llabel{e:xt+k-norm1}
\sumz s{t-1} |a_s| &= L_a\\
\ve{A^k v + \sumo jk A^{k-j} \pa{\sumz s{t-1} -a_s\xi_{s+j} + \xi_{t+j}}} &\le
\ve{A^kv} + \pa{\sumz s{k-1}f_{A,K}(s)}\pa{\pa{\sumz s{t-1} |a_s|}+1 } \\
&\le \ve{A^kv} + \pa{\sumz s{k-1}f_{A,K}(s)} (L_a+1).
\llabel{e:xt+k-norm2}
\end{align}
By Lemma~\ref{l:ST-card}, $|I_T|\le d\log_2M$, so for $1\le t\le T$, there exists $k\le |I_T|\le k'=\fl{d\log_2(M)}$ such that 
either
\begin{itemize}
\item
$t-k\nin I_T$, so $\ve{x_{t-k}}_{\De_{t-k-1}}\le \Lad$, or
\item
$t-k=0$, so $x_{t-k}=x_0\in K_0$. 
\end{itemize}
In either case, we can write $x_{t-k} = \sumz s{t-k-1} a_sx_s + v$ with $\sumz s{t-k-1}|a_s|\le L_a:=\fc{2}{\ln 2}d$ and 
$v\in K_0\cup B_{\Lad}^d$. 
Then by~\eqref{e:xt+k-norm1} and~\eqref{e:xt+k-norm2}, we can write $x_t$ as
\begin{align}
x_t &= \sumz s{t-1} a_s' x_s + v'
\end{align}
with 
\begin{align}
\sumz s{t-1}|a_s'|&\le \Lad\\
\ve{v}_2 &
\le \ve{A^kv} + \pa{\sumz s{k'-1}f_{A,K}(s)} (L_a+1)\\
&\le \Lvg.
\end{align}
\end{proof}

\subsubsection{Bound in terms of Jordan form}

\begin{lem}\llabel{l:J-power}
Suppose that $A$ has spectral radius $\le 1$ and $A=SJS^{-1}$ with $\ve{S}_2\ve{S^{-1}}_2\le C_A$ and $J$ has Jordan blocks of rank $\le r$. Then
\begin{align}\llabel{e:jrA}
\ve{A^k}_2 &\le  \binom{k+r}{r-1} C_A\le 
(k+1)^{r-1}C_A\\
\llabel{e:jrB}
\sumz sk \ve{A^s}_2 &\le \binom{k+r+1}{r}C_A \le (k+1)^{r}C_A
\end{align}
\end{lem}
\begin{proof}
Let 
$J_{\la, r}=\pa{\begin{smallmatrix}\lambda & 1\\
 & \lambda & \ddots\\
 &  & \ddots & 1\\
 &  &  & \lambda
\end{smallmatrix}}\in \R^{r\times r}$. Then entrywise, $J_{\la,r}^k$ has absolute value at most 
$J_{1 ,r}^k=\pa{\begin{smallmatrix}1 & \binom{k}{1} & \cdots & \binom{k}{r-1}\\
 & 1 & \ddots & \vdots\\
 &  & \ddots & \binom{k}{1}\\
 &  &  & 1
\end{smallmatrix}}$. We have for $k\ge 1$ that
\begin{align}
\ve{J_{\la,r}^k}_2&\le
\ve{J_{\la,r}^k}_{1,1}\le 
\ve{J_{1,r}^k}_{1,1}\\
 &\le \sumz q{r-1} \sumz jq \binom kj\le 
 \sumz q{r-1}\sumz j{q}\binom{k+j-1}j
\le \sumz q{r-1} \binom{k+q}{q}\\
&\le \binom{k+r}{r-1} \le \fc{(k+1)^r}{(r-1)!} \cdot \fc{(k+r)\cdots (k+2)}{(k+1)^r} \le
(k+1)^{r-1}
\llabel{e:jr1}
\end{align}
By decomposing into Jordan blocks, we find that 
\begin{align}\llabel{e:jr2}
\ve{A^k}_2\le \ve{S}_2\ve{J_{\la,r}^k}_2\ve{S^{-1}}_2\le C_A\ve{J_{\la,r}^k}_2.
\end{align}
Together~\eqref{e:jr1} and~\eqref{e:jr2} give~\eqref{e:jrA}. Summing over $k$ gives
~\eqref{e:jrB}.
\end{proof}


\subsection{Concluding the states are anomaly-free}

We need the following number-theoretic lemma.
\begin{lem}[{\citealt[Theorem 6.3]{nath}}]\llabel{l:prime-prod}
There exists a constant $C>0$ such that for all $x\ge 2$,
\begin{align*}
\prod_{1<p\le x,p\text{ prime}} p &\ge x^{Cx}.
\end{align*}
\end{lem}
The next lemma, Lemma~\ref{l:many}, is technical to state, so we first explain the idea. Recall that we want the sequence $(x_t)$ to be anomaly-free in order to apply the bound on regret in Theorem~\ref{thm:ols-regret}, i.e., we want a lot of $x_s$'s such that $|w^\top x_s|$ is large relative to $w^\top x_T$.
Lemma~\ref{l:xt-l1-comb} together with Lemma~\ref{c:exists-large} only gives a single large $|w^\top x_s|$, so the idea is to apply the lemmas to the dynamical systems $x_t\approx A^p x_{t-p}$ for different values of $p$. We choose prime $p$'s and use Lemma~\ref{l:prime-prod} to control overcounting.

We separate out the core argument from the numerical bounds, so Lemma~\ref{l:many} is stated in terms of the functions $L_v$ and $Q$. In the lemma, $L_v(p)$ tracks how large the residual term in~\eqref{e:Lvg} can get. This is the term left over when writing $x_T$ as a linear combination of previous terms. We will instantiate $L_v(p) = O((pk')^r(d+m))$ in the proof of Theorem~\ref{t:main} (ignoring constants related to the LDS). In order to apply Lemma~\ref{c:exists-large}, we would like to only consider $p$ small enough so that $2L_v(p)\le |w^\top x_T|$, or $(pk')^r(d+m)=O(|w^\top x_T|)$. Roughly, this is true for $p = O\pa{\rc{k'} \pf{|w^\top x_T|}{d+m}^{1/r}}$. This RHS is approximately the $Q$ appearing in the lemma, $Q(x) \approx \rc{k'}\pf{x}{d+m}^{1/r}$. For each of these primes $p\le Q(|w^\top x_T|)$, we obtain some $x_s$ such that $|w^\top x_s|$ is large. This gives almost $Q(|w^\top x_T|)$ many large $x_s$'s. We lose logarithmic factors, since we restrict to primes and account for overcounting. 

The key dependence to note here is that on $r$, the size of the largest Jordan block: the residual $L_v(p)$ has $r$th power dependence, so the number of large $|w^\top x_s|$'s has $\rc r$-power dependence. Any constant power gives a sublinear regret bound in Theorem~\ref{thm:ols-regret}, with a smaller value of $\al=\rc r$ giving a worse regret bound.
\begin{lem}\llabel{l:many-large}\llabel{l:many}
Let $T>0$. 
Suppose that $\ve{x_0}\le C_0$ and $x_t=Ax_{t-1}+\xi_t$ with 
$\xi_t\in K$ for $1\le t\le T$. Let $k'=\kp$ and let
\begin{align}
L_v(p) &= \Lp
\end{align}
Suppose the following hold.
\begin{itemize}
    \item (Bound on $x_t$) $M\ge 1$, $\ve{x_t}\le M$ for $1\le t\le T$.
    \item ($|w^\top x_T|$ is large enough) $c\ge \fc{4}{\ln 2}dC_0$, and $|w^\top x_T|> c$.
    \item ($Q$ gives the range of $p$ that are ``good" for Lemma~\ref{c:exists-large}) $Q(x)$ is a function such that whenever $x> c$, then $2\le Q(x)$ and for all $p\le Q(x)$, we have $x\ge 2L_v(p)$. 
\end{itemize}
Then 
there are at least $\Om\pf{Q(|w^\top x_T|)}{\ln T}$
values of $s$, $0\le s<t$, such that 
\begin{align}
\ab{w^\top x_s} \ge \fc{\ln 2}{4d}|w^\top x_T|.
\end{align}
\end{lem}
\begin{proof}
For $p\le T$, we apply 
Lemma~\ref{c:exists-large} with the bounds from Lemma~\ref{l:xt-l1-comb} with the sequence $x_{T\bmod p},\ldots, x_{T-p},x_T$. 
This satisfies the conditions of Lemma~\ref{l:xt-l1-comb} with
\begin{align}
x_{t}&=A' x_{t-p} + \xi_{t}^{(p)}\\
A'&=A^p\\
x_{T\bmod p}&\in K_0':= A^{T\bmod p}B_{C_0}^d
+A^{(T\bmod p)-1}K+\cdots +K\\
\xi_t^{(p)}=\sumo jp A^{p-j}\xi_{t-p+j}&\in K':= \sumo jp A^{p-j} K.
\end{align}
Using this, we write $f_{A', K_0'\cup B_{\fc{2}{\ln 2}d}^d}(k)$ in terms of $f_A$ and $f_{A,K}$. We upper bound by taking the maximum inside the sum, to get the following.
\begin{align}
&\max_{0\le k\le k'} f_{A', K_0'\cup B_{\fc{2}{\ln 2}d}^d}(k)\\
&=
\max
\Bigg\{
\maxr{0\le k\le k'}{v\in 
B_{C_0}^d, \xi_1,\ldots, \xi_{T\bmod p}\in K}
\pa{A^{(T\bmod p)+kp}v+\sumo j{T\bmod p} A^{(T\bmod p)-j+kp}\xi_j},\\
\nonumber
&\quad \quad \max_{0\le k\le k'} \ve{A^{kp}}_2\fc{2}{\ln 2}d
\Bigg\}\\
&\le 
\max_{0\le k\le p(k'+1)-1} f_A(k)\maxdc
+ \sumz k{p(k'+1)-2} f_{A,K}(k).
\llabel{e:lv-bd1}
\end{align}
Similarly,
\begin{align}\sumz k{k'-1} f_{A',K'}(k)
=\sumz k{k'-1} \max_{\xi_0,\ldots, \xi_{p-1}\in K}
\ve{\sumo jp A^{pk+p-j}\xi_j}
 &\le
\sumz k{k'p-1} \max_{\xi\in K} \ve{A^k\xi}=
\sumz k{pk'-1}f_{A,K}(k).
\llabel{e:lv-bd2}
\end{align}
Then for the sequence $x_{T\bmod p},\ldots, x_T$, the $L_v$ in Lemma~\ref{l:xt-l1-comb} equals
\begin{align}
&\max_{0\le k\le k'} f_{A',K_0'\cup \Bd}(k)
+\pa{\sumz k{k'-1} f_{A',K'}(k)}\pa{\Lad+1}\\
&\le \Lp=:L_v(p)
\end{align}
using~\eqref{e:lv-bd1} and~\eqref{e:lv-bd2}.
Apply Lemma~\ref{c:exists-large} to get that there exists $k\in \N$ such that 
\begin{align}\llabel{e:mwlp}
|w^\top x_{T-kp}| \ge \fc{|w^\top x_T|-L_v(p)}{\fc{2}{\ln 2}d}.
\end{align}
When $|w^\top x_T|>c$,  and $p\le Q(|w^\top x_T|)$, by assumption $|w^\top x_T|\ge 2L_v(p)$ and hence
$\fc{|w^\top x_T|-L_v(p)}{\fc{2}{\ln 2}d} \ge \fc{|w^\top x_T|}{\fc{4}{\ln 2}d}$. 
Then equation~\eqref{e:mwlp} becomes: there exists $k\in \N$ such that 
\begin{align}\llabel{e:mwlp2}
|w^\top x_{T-kp}| \ge \fc{\ln 2}{4d} |w^\top x_T|
\end{align}
Put another way, there exists $s$ such that $|w^\top x_s|\ge  \fc{\ln 2}{4d} |w^\top x_T|$ and $p$ divides $T-s$. Let $Q=Q(|w^\top x_T|)$. This means that $\prod_{|w^\top x_s|\ge\fc{\ln 2}{4d}|w^\top x_T|}(T-s)$ has to be divisible by all primes $p\le \min\{Q,T\}=Q$. (Note $Q\le T$ because otherwise,~\eqref{e:mwlp2} would hold for $p=T<Q$ and we have $C_0\ge |w^\top x_0|> \fc{\ln 2}{4d}|w^\top x_T|\ge \fc{\ln 2}{4d}c$, contradiction.)
Let $N=\ab{\set{s}{|w^\top x_s|\ge \fc{\ln 2}{4d}|w^\top x_T|}}$. By Lemma~\ref{l:prime-prod}, for some $C'>0$,
\begin{align}
T^N \ge \prod_{|w^\top x_s|\ge  \fc{\ln 2}{4d} |w^\top x_T|}(T-s)\ge \prod_{\text{prime }p\le Q}p \ge Q^{C'Q}.
\end{align}
Thus
\begin{align}
N&\ge \fc{C'Q\ln Q}{\ln T} \ge \Om\pf{Q}{\ln T}.
\end{align}
\end{proof}

\subsection{Finishing the proof}

\begin{lem}\llabel{l:Mbd}
Assume Assumption~\ref{asm:1}. Then 
\begin{align}
\max_{0\le k\le T} \ve{x_k} &\le \Mbd.
\end{align}
\end{lem}
\begin{proof}
Using Lemma~\ref{l:J-power} we have
\begin{align*}
x_t &= A^t v + \sumz k{t-1} A^{k}Bu_{t-1-k} + \sumz k{t-1} A^k \xi_{t-k}\\
\max_{0\le k\le T} \ve{x_k}
&\le \max_{0\le k\le T} \ve{A^k}C_0 + \sumz k{T-1}\ve{A^kB} C_u + \sumz k{T-1} \ve{A^k}C_\xi\\
&\le (T+1)^{r-1} C_AC_0 + C_AC_BC_u T^r + C_AC_\xi T^r.
\end{align*}•
\end{proof}

\begin{proof}[Proof of Theorem~\ref{t:main}]
Let $A'=\smatt ABOO$. 
We apply Lemma~\ref{l:many-large} to the system
\begin{align}
\coltwo{x_t}{u_t} &= A'\coltwo{x_{t-1}}{u_{t-1}} + \coltwo{\xi_t}{u_t}
\llabel{e:ABOO}
\end{align}
with $\scoltwo{\xi_t}{u_t}\in K:=B_{C_\xi}^d \opl B_{C_u}^d$ for $1\le t\le T$. 
Let $k'=\kdm$ and $C=C_A(C_\xi + C_BC_u)$.
Note that
\begin{align}
\nonumber
f_{A', K}(k)
&= 
\max_{\scoltwo{\xi}{u}\in K} \pa{\ve{A^k\xi}_2 + \begin{cases}\ve{A^{k-1}Bu}_2,&k\ge 1\\
C_u,&k=0
\end{cases}}\\
\sumz k{p(k'+1)-2} f_{A',K}(k)
&\le \sumz k{p(k'+1)-2} (f_A(k) C_\xi + f_A(k)C_BC_u) + C_u
\llabel{l:sum-fA'K}
\end{align}
Then using Lemma~\ref{l:J-power}, $L_v(p)$ in Lemma~\ref{l:many-large} equals
\begin{align}
\nonumber
L_v(p)&=
{\max_{0\le k\le p(k'+1)-1} f_{A'}(k) \maxdmc +  \pa{\sumz k{p(k'+1)-2}f_{A',K}(k)} \pa{\Ladm+2}}\\
\nonumber
&\le 
\max_{0\le k\le p(k'+1)-1} f_{A'}(k) \maxdmc \\
\nonumber
&\quad +  \pa{\sumz k{p(k'+1)-2} (f_A(k) C_\xi + f_A(k)C_BC_u) + C_u} \pa{\Ladm+2}\\
\nonumber
&\le (p(k'+1))^{r-1} \maxdmc +  [C(p(k'+1))^r+C_u] \pa{\Ladm+2}\\
&\le p^r (k'+1)^r \pa{\Ladm(C+1)+C_0+2} + C_u\pa{\Ladm+2}.
\llabel{e:lvp}
\end{align}
Let 
\begin{align*}
Q(x):&=
\pf{x-2C_u(\Ladm+2)}{2 C(k'+1)^r \pa{\Ladm(C+1)+C_0+2}}^{\rc r}\\
c:&= \max\bc{\clbd,\fc{4d C_0}{\ln 2}}.
\end{align*}
It is easy to check that when $x\ge c$, then $2\le Q(x)$ and for all  $p\le Q(x)$, $x\ge 2L_v(p)$. 
Moreover, for $x\ge c$, we have $x\ge 4C_u\pa{\Ladm+2}$, so $Q(x)\ge \pf{x}{4C(k'+1)^r\pa{\Ladm(C+1)+C_0+2}}^{\rc r}$.
By Lemma~\ref{l:many-large}, the states are $(c, c_1, c_2, \al)$-anomaly-free, with
\begin{align*}
\al&=\rc r\\
c&=\max\bc{\clbd,\fc{4d C_0}{\ln 2}}\\
c_1 &=\Te\pf{1}{C^{\rc r}(k'+1)\pa{\Ladm(C+1)+C_0}^{\rc r}\ln T}\\
c_2 &= \fc{\ln 2}{4d},
\end{align*}
satisfying the hypothesis of Theorem~\ref{thm:ols-regret}.
By Theorem~\ref{thm:ols-regret}, the regret is
\begin{align*}
R_T&\le \mu R^2 + 
O\Bigg(
(d+m)^3\maxprederrorJAT^2\\
&\quad \cdot 
\ln \pa{1+\fc{TM^2}{d+m}}
\Bigg).
\end{align*}
Plugging in $\mu = T^{\fc {2r+1}{2r+2}}$, and using the bound $M\le \Mbd$ from Lemma~\ref{l:Mbd}, we obtain the theorem. Note that the only terms with an $r$th power are the two terms involving $c$. The dependence on $T$ of the larger such term is $\frac{T}{\mu} = T^{\rc{2r+2}}$, hence the additive term in~\eqref{e:main}.

We now simplify the bound when all the constants are $O(1)$ and $\ve{(A,B)}_F=O(d+m)$.
We have $c_1=\Om\pf{1}{(d+m) r\ln T (d+m)^{\rc r} \ln T}$ and $c_2 = \Om\prc{d}$. 
Using~\eqref{e:ols-simp-R} in Theorem~\ref{thm:ols-regret}, with $\mu = \fc{(d+m)^{\fc{\al+2}{\al+1}}T^{\fc{\al+2}{2(\al+1)}}}{c_1^{\rc{\al+1}}c_2^{\fc2{\al+1}}}$, we obtain the bound (as $T\to \iy$)
\begin{align}
\nonumber
    R_T &= 
    O\pf{(d+m)^2 (d+m)^{\fc{r}{r+1}}T^{\fc{2r+1}{2r+2}}r\ln T}{c_1^{\fc{r}{r+1}}c_2^{\fc{2r}{r+1}}}\\
    \nonumber
    &=
    O\pa{(d+m)^{2+\fc{r}{r+1}}
T^{\fc{2r+1}{2r+2}}r\ln T \cdot 
\pa{(d+m)^{\fc{r+1}{r}}r(\ln T)^2}^{\fc{r}{r+1}} d^{\fc{2r}{r+1}}}\\
\label{e:before-last-step}
&= O\pa{
(d+m)^{2+\fc{r}{r+1}}d^{1+\fc{r-1}{r+1}} (d+m) T^{\fc{2r+1}{2r+2}} r^2 (\ln T)^3 
}\\
&= O\pa{
(d+m)^{4}d^{2} r^2 T^{\fc{2r+1}{2r+2}} (\ln T)^3 
}.
\nonumber
\end{align}
Note that Corollary~\ref{c:main} follows from plugging $r=1$ into~\eqref{e:before-last-step}.
\end{proof}

\section{Proof of Theorem~\ref{t:main-stoch} (fully observed system, stochastic noise)}
\llabel{s:proof-stoch}


\begin{lem}\llabel{l:stoch}
Let $v_t,w_t\in \R^d$ for $0\le t\le T-1$. 
Suppose that the following hold.
\begin{itemize}
\item
For any $0\le t\le T-1$, there exist $a_s\in \R$ such that 
\begin{align}\llabel{e:vt-lc}
v_t &= \sumz s{t-1} a_s v_s + v,&
\sqrt{\sumz s{t-1} a_s^2} &\le L_a,&
\ve{v}_2 &\le L_v.
\end{align}
\item
$w_t = Av_t + \ep_{t+1} + \xi_{t+1}$ where $\ve{A}_2\le R$, $\ve{\ep_t}\le C_\ep$, and 
$\xi_t$, $1\le t\le T$ are random variables such that 
$\xi_{t+1}|\xi_1,\ldots, \xi_t$ is mean 0 and $C_\xi$-subgaussian for $0\le t\le T-1$.
\item
$\ve{v_t}\le g(t)$ for $0\le t\le T-1$. 
\end{itemize}
Let $L'>0$, and let 
\begin{align}
\llabel{e:ep-net}
\ep_{\textup{net}}&= \fc{1}{\sumz s{T-1}g(s)}\\
\llabel{e:Rz}
R_z&= \mu^{-\rc 2} \pa{L_a^2 + \fc{L_v^2}{\mu}}^{\rc 2} d^{\rc 2}\\
\llabel{e:Sb}
S_b & =\sqrt 2\pa{\pa{L_a^2 + \fc{L_v^2}{\mu}}d^2 + 1}^{\rc 2}
\\
\llabel{e:L}
L&=\sqrt2 C_\xi S_b \pa{d \ln \pa{1+\fc{2R_z}{\ep_{\textup{net}}}} + \ln \pf{T}{\de}}^{\rc 2}.
\end{align}
Then
\begin{align}
\Pj\pa{
\max_{0\le t\le T-1} \ve{A_tv_t-v_{t+1}}_2 \le L + 
(\mu^{\rc 2}R + C_\ep d^{\rc 2} \sqrt{T})\pa{L_a^2 + \fc{L_v^2}{\mu}}^{\rc 2} d^{\rc 2} +
2C_\xi + C_\ep
}&\ge 1-\de.
\end{align}
\end{lem}
The reason for the choice of values of $\ep_{\textup{net}}$, $R_z$, $S_b$, and $L$ can be seen in~\eqref{e:choose-ep},~\eqref{e:zt},~\eqref{e:choose-Sb}, and~\eqref{e:stoch-bd1}, respectively.

We note several lemmas we will need.
\begin{lem}\llabel{l:vSvd}
Let $v_s\in \R^d$ for $1\le s\le t$. Suppose that $\sumo st v_sv_s^\top \preceq \Si$. Then 
\begin{align}
\sumo s{t} v_s^\top \Si^{-1} v_s &\le d.
\end{align}
\end{lem}
\begin{proof}
We can choose $v_s'\in \R^d$, $1\le s\le t'$ such that $\sumo st v_sv_s^\top + \sumo s{t'} v_s'v_s^{\prime\top}=\Si$. Then
\begin{align}
\sumo s{t} v_s^\top \Si^{-1} v_s
+ 
\sumo s{t'} v_s^{\prime\top} \Si^{-1} v_s'
&= \Tr\pa{
\Si^{-1} \pa{\sumo st v_s v_s^\top + \sumo s{t'} v_s'v_s^{\prime\top}}
}=\Tr(I_d)=d\\
\implies
\sumo s{t} v_s^\top \Si^{-1} v_s &= d-\sumo s{t'} v_s^{\prime\top} \Si^{-1} v_s'\le d.
\end{align}
\end{proof}
\begin{lem}[\citealt{vershynin2010introduction}, Lemma 5.2]
\llabel{l:ep-net}
There is an $\ep$-net of $B_1^d$ of size $\pa{1+\fc{2}{\ep}}^d$.
\end{lem}
\begin{lem}[{Generalization of Azuma's Inequality~\cite[Lemma 4.2]{simchowitz2018learning}}]
\llabel{l:azuma}
Let $\{\cal F_t\}_{t\ge 0}$ be a filtration, and $\{Z_t\}_{t\ge 1}$ and $\{W_t\}_{t\ge 1}$ be real-valued processes adapted to $\cF_t$ and $\cF_{t+1}$ respectively. Moreover, assume $W_t|\cF_t$ is mean 0 and $\si^2$-sub-Gaussian. Then for any positive real numbers $L$ and $\be$,
\begin{align*}
\Pj\ba{\bc{\sumo tT W_t Z_t\ge L} \cap \bc{\sumo tT Z_t^2 \le \be}} &\le \exp\pa{-\fc{L^2}{2\si^2\be}}
\end{align*}
\end{lem}
\begin{proof}[Proof of Lemma~\ref{l:stoch}]
Let 
\begin{align}
\Si_t &= \mu I_d + \sumz s{t-1} v_sv_s^\top .
\end{align}
By Lemma~\ref{l:ls-calc},
\begin{align}
A_tv_t - w_t 
&= \sumz s{t-1} (\ep_{s+1}+\xi_{s+1})v_s^\top \Si_t^{-1} v_t
- \mu \Si_t^{-1} v_t - \ep_{t+1} - \xi_{t+1}.
\llabel{e:Av-w}
\end{align}

\paragraph{Bounding $\sumz s{t-1} \xi_{s+1}v_s^\top \Si_t^{-1} v_t$.}

Let $b_s = v_s^\top \Si_t^{-1} v_t$. 

To use Azuma's inequality, we need to bound the following. Using Cauchy-Schwarz ($\an{v,w}^2\le \ve{v}^2\ve{w}^2$),
\begin{align}
\sumz s{t-1} b_s^2 &\le \sumz s{t-1} \ve{v_s^\top \Si_t^{-\rc2}}_2^2 \ve{\Si_t^{-\rc 2} v_t}_2^2\\
&=\pa{\sumz s{t-1} v_s^\top \Si_t^{-1} v_s}\ve{\Si_t^{-\rc 2} v_t}_2^2
\llabel{e:bs2}
\end{align}
Choosing $a_s$ and $v$ as in~\eqref{e:vt-lc}, 
\begin{align}
\ve{\Si_t^{-\rc 2} v_t}_2^2&=
\ve{\Si_t^{-\rc 2} \pa{\sumz s{t-1}a_sv_s + v}}_2^2=\ve{\sumz s{t-1} a_s \Si_t^{-\rc 2} v_s + \Si_t^{-\rc 2} v}_2^2\\
&\le \pa{\sumz s{t-1} a_s^2 + \fc{\ve{v}^2}{\mu}} \pa{\sumz s{t-1} \ve{\Si_t^{-\rc 2}v_s}_2^2 +\fc{\mu}{\ve{v}^2} \ve{\Si_t^{-\rc 2}v}_2^2} &\text{by Cauchy-Schwarz}\\
&\le \pa{L_a^2 + \fc{L_v^2}{\mu}} \pa{\sumz s{t-1} v_s^\top \Si_t^{-1} v_s +{\mu} \fc{v^\top }{\ve{v}}\Si_t^{-1} \fc{v}{\ve{v}}}.
\llabel{e:Sv}
\end{align}
Now $v_sv_s^\top + \mu \fc{vv^\top}{\ve{v}^2}\preceq \Si_t$ satisfies the hypothesis of Lemma~\ref{l:vSvd}, so
\begin{align}
\sumz s{t-1} v_s^\top \Si_t^{-1} v_s +{\mu} \fc{v^\top }{\ve{v}}\Si_t^{-1} \fc{v}{\ve{v}} &\le 
d 
\llabel{e:vsSvs}
\end{align}
and so~\eqref{e:Sv} gives
\begin{align}
\ve{\Si_t^{-\rc 2}v_t}^2 &\le \pa{L_a^2 + \fc{L_v^2}{\mu}}d.
\llabel{e:Sit-vt}
\end{align}
From~\eqref{e:bs2},~\eqref{e:vsSvs}, and~\eqref{e:Sit-vt}, we get
\begin{align}\llabel{e:bs2-2}
\sumz s{t-1} b_s^2 &\le \pa{L_a^2 + \fc{L_v^2}{\mu}} d^2.
\end{align}•

Let $z_t = \Si_t^{-1} v_t$. 
By~\eqref{e:Sv} and~\eqref{e:Sit-vt},
\begin{align}\llabel{e:zt}
\ve{z_t} &= \ve{\Si_t^{-\rc 2}}_2 \ve{\Si_t^{-\rc2} v_t}_2\le 
\mu^{-\rc 2}\pa{L_a^2 + \fc{L_v^2}{\mu}}^{\rc 2}d^{\rc 2}=R_z
\end{align}
so $z_t\in B_{R_z}^d$. 
By Lemma~\ref{l:ep-net}, there is an $\ep_{\textup{net}}$-net $\cal N$ of $B_{R_z}^d$ of size $\pa{1+\fc{2R_z}{\ep_{\textup{net}}}}^d$. Since $z_t\in B_{R_z}^d$, there exists $\ve{\De z}\le \ep_{\textup{net}}$ such that $z_t+\De z\in \cal N$.
Note that
\begin{align}
\sumz s{t-1} |v_s^\top (z_t+\De z)|^2 &\le 
2\pa{\sumz s{t-1} |v_s^\top z_t|^2 + |v_s^\top \De z|^2}
\le 2\pa{\pa{\sumz s{t-1} b_s^2 }d^2 + \sumz s{t-1} g(s)^2 \ep_{\textup{net}}^2}\\
&\le 2\pa{\pa{L_a^2 + \fc{L_v^2}{\mu}}d^2 + 1}=:S_b^2.
\llabel{e:choose-Sb}
\end{align}
For a fixed $z$, consider the event $E_z = \{ 
\sumz s{t-1} |y_s^\top z|^2 > S_b^2 \text{ or }
\sumz s{t-1} \xi_{s+1} y_s^\top z \le L \}$. Then by the generalization of Azuma's inequality~\ref{l:azuma},
\begin{align}
\Pj(E_z^c) &\le \Pj\pa{
\sumz s{t-1} |y_s^\top z|^2 \le S_b^2 \text{ and }
\sumz s{t-1} \xi_{s+1} y_s^\top z > L
}\\
&\le \exp\pa{-\fc{L^2}{2C_\xi^2 S_b^2}}.
\end{align}
Now we use the triangle inequality to bound the sum by the maximum value on the $\ep_{\textup{net}}$-net.
\begin{align}
\ve{\sumz s{t-1} \xi_{s+1} v_s^\top z} &\le 
\ve{\sumz s{t-1} \xi_{s+1} v_s^\top (z+\De z)}+ 
\ve{\sumz s{t-1} \xi_{s+1} v_s^\top (\De z)}\\
&\le \maxr{z'\in \cal N}{\sumz s{t-1} |v_s^\top z'|^2\le S_b^2} \ab{\sumz s{t-1} \xi_{s+1} v_s^\top z'}
 + C_\xi \sumz s{t-1} g(s)\ep_{\textup{net}}
 \llabel{e:choose-ep}\\
 &=  \maxr{z'\in \cal N}{\sumz s{t-1} |v_s^\top z'|^2\le S_b^2} \ab{\sumz s{t-1} \xi_{s+1} v_s^\top z'}
 + C_\xi
\end{align}
Under the event $\bigcap_{z'\in \cal N}E_{z'}$, we have that the maximum above is $\le L$. Thus
\begin{align}
\Pj\pa{
\ve{\sumz s{t-1} \xi_{s+1} v_s^\top z}
>C_\xi+L
}&\le \Pj\pa{\bigcup_{z'\in \cal N} E_{z'}^c}\\
&\le \pa{1+\fc{2R_z}{\ep_{\textup{net}}}}^d \exp\pa{-\fc{L^2}{2C_\xi^2 S_b^2}}\le \fc{\ep}T.
\llabel{e:stoch-bd1}
\end{align}
by the choice of $L$.

\paragraph{Bounding $\sumz s{t-1} \ep_{s+1}v_s^\top \Si_t^{-1} v_t$.}
A crude bound suffices here. We have by~\eqref{e:bs2-2} that
\begin{align}
\sumz s{t-1} \ep_{s+1}v_s^\top \Si_t^{-1} v_t
&\le C_\ep \sumz s{t-1} |v_s^\top (\Si_t)^{-1} v_t|\\
&\le C_\ep \sqrt{t} \sqrt{\sumz s{t-1} |v_s^\top (\Si_t)^{-1} v_t|^2}\\
&\le C_\ep \sqrt{T} \pa{L_a^2 + \fc{L_v^2}{\mu}}^{\rc 2}d.
\llabel{e:stoch-bd1b}
\end{align}

\paragraph{Bounding $\mu A\Si_t^{-1} v_t$.} We have by~\eqref{e:zt} that 
\begin{align}\llabel{e:stoch-bd2}
\ve{\mu A\Si_t^{-1}v_t} \le \mu AR_z = \mu^{\rc 2}R\pa{L_a^2 + \fc{L_v^2}{\mu}}^{\rc 2} d^{\rc 2}
\end{align}

From~\eqref{e:Av-w},~\eqref{e:stoch-bd1},~\eqref{e:stoch-bd1b}, and~\eqref{e:stoch-bd2}, and noting that $\ve{\ep_{t+1}+\xi_{t+1}}\le C_\xi+C_\ep$, we get that 
\begin{align}
\Pj\pa{\ve{A_tv_t-w_t}> L + (\mu^{\rc 2}R + C_\ep d^{\rc 2}\sqrt{T})\pa{L_a^2 + \fc{L_v^2}{\mu}}^{\rc 2} d^{\rc 2} + 2C_\xi + C_\ep}&\le \fc{\ep}{T}.
\end{align}
Union-bounding over $0\le t\le T-1$ finishes the proof.
\end{proof}


\begin{proof}[Proof of Theorem~\ref{t:main-stoch}]
We apply Lemma~\ref{l:stoch} to the system~\eqref{e:ABOO}. 
Let $k'= \kdm$ and $C=C_A(C_BC_u+C_\xi)$.  
By~\eqref{e:lvp} with $p=1$, we can write $x_t=\sumz st a_sx_s+v$ with $\sumz s{t-1} |a_s|\le L_a$ and $\ve{v}_2\le L_v$, where
\begin{align}
L_a &= \Ladm\\
L_v &=  (k'+1)^r \pa{\Ladm(C+1)+C_0+2} + C_u\pa{\Ladm+2}.
\end{align}
using~\eqref{l:sum-fA'K} in the last step.

Then Lemma~\ref{l:stoch} is satisfied with these values of $L_a$, $L_v$, and (by Lemma~\ref{l:Mbd}), $g(t) = (t+1)^{r-1} C_AC_0 + t^r C_A(C_BC_u+C_\xi)$. 
Defining $\ep_{\textup{net}}$, $R_z$, $S_b$, $L$ as in Lemma~\ref{l:stoch}, we get that
\begin{align}
\Pj\pa{
\max_{0\le t\le T-1} \ve{A_tv_t-v_{t+1}}_2 \le L + 
\mu^{\rc 2}\pa{L_a^2 + \fc{L_v^2}{\mu}}^{\rc 2} d +
2C_\xi
}&\ge 1-\ep.
\end{align}
Then with probability $1-\ep$, we have by Theorem~\ref{t:orr} that
\begin{align}
R_T(A,B) &\le 
\mu R^2 + 
\pa{L + 
\mu^{\rc 2}\pa{L_a^2 + \fc{L_v^2}{\mu}}^{\rc 2} d +
2C_\xi}^2 (d+m) \ln \pa{1+\fc{TM^2}{d+m}}
\end{align}
where $L_a,L_v,L,M$ have the desired parameter dependences; take $\mu=1$.
\end{proof}

\section{Proof: Partially observable system, stochastic setting}
\llabel{s:proof-hidden}

\subsection{Learning the steady-state Kalman filter}
\label{subsec:kalman}

In the prediction problem, at each time step $t$, we have observed $y_0,\ldots, y_t$ and $u_0, \ldots, u_t$, and are asked to predict $y_{t+1}$. Note that this does not immediately fit in the framework for online least squares, 
because $y_{t+1}$ is a linear function of the unobserved $x_t$ (the latent state) plus noise. However, as we will see, we can still place it in this framework if we use an linear autoregressive estimator.

The Kalman filter~\citep{kalman1960new,kamen1999introduction} gives the optimal 
linear estimator in the case that the parameters of the LDS are known, the initial state is drawn from a known Gaussian distribution $x_0\sim N(x_0^-, \Si_0)$, and the noises are independent mean-zero with known covariances. 
We can compute matrices $A_{\KF}^{(t)}$, $B_{\KF}^{(t)}$, and $C_{\KF}^{(t)}=C$ such that the optimal linear estimate of the latent state $\wh h_t$ and the observation $\wh y_t$ are given by a time-varying LDS (taking the $y_t$ as feedback) with those matrices: 
\begin{align}
\label{e:KF}
x^-_t &=A_{\KF}^{(t)}x^-_{t-1}  +  B_{\KF}^{(t)} \coltwo{u_{t-1}}{y_{t-1}}\\
\wh y_t &= C_{\KF}^{(t)}x^-_t.
\label{e:CKF}
\end{align}
This is known as the predictor form of the system~\cite{qin2006overview}.
We will denote $B_{KF}^{(t)} = (B_{KF,u}^{(t)}\;B_{KF,y}^{(t)})$, where $B_{KF,u}^{(t)}$ and $B_{KF,y}^{(t)}$ are the submatrices acting on $u_{t-1}$ and $y_{t-1}$, respectively.
In the case that the noises are iid Gaussian, i.e. $\xi_t\sim N(0,\Si_x)$ and $\eta_t\sim N(0,\Si_y)$, $x^-_t$ and $\wh y_t$ are the maximum a posteriori (MAP) estimators, and the actual hidden state $x_t$ and the observation $y_t$ are Gaussians when conditioned on $\cal F_{t-1}=\si(y_0,\ldots, y_{t-1})$ (the observations up to time $t-1$): $x_t|\cal F_{t-1} \sim N(x^-_{t-1}, \Si_{\KF,x}^{(t)})$ and $y_t|\cal F_{t-1} \sim N(\wh y_t,\Si_{\KF,y}^{(t)})$ for some covariance matrices $\Si_{\KF,x}^{(t)}$, $\Si_{\KF,y}^{(t)}$. Our goal is to predict as well as the Kalman filter without knowing the parameters of the original LDS. 

If the original system satisfies Assumption~\ref{asm:obs} (most notably, it is observable and the noise is i.i.d.), taking $t\to \iy$, 
the matrices $A_{\KF}^{(t)}$, and $B_{\KF}^{(t)}$ approach certain fixed matrices $A_{\KF}$ and $B_{\KF}$, 
and the covariance matrices $\Si_{\KF,x}^{(t)}$ and $\Si_{\KF,y}^{(t)}$ approach fixed matrices $\Si_{\KF,x}$ and $\Si_{\KF,y}$ \citep{harrison1997convergence,anderson2012optimal}. 
In the Gaussian case, at steady-state, the actual hidden state $x_t$ and observation $y_t$ will be distributed as $x_t|\cal F_{t-1} \sim N(x^-_t, \Si_x)$ and $y_t|\cal F_{t-1} \sim N(\wh y_t,\Si_y)$. To simplify the problem, we will assume that the LDS starts with the steady-state covariance\footnote{Note that steady-state refers to the covariance of the $x_t|\cal F_{t-1}$ and $y_t|\cal F_{t-1}$ being constant, rather than the distribution of $x_t|\cal F_{t-1}$ and $y_t|\cal F_{t-1}$ being constant. If the system is not strictly stable, it does not have a steady state in the open-loop setting, because $x_t$ will diverge.}, so that the steady-state Kalman filter is the optimal filter for all time.




As before, our task is to predict $\wh y_{t+1}$ at time step $t$. The regret is now defined by
\begin{align}\label{e:RTABC}
    R_T(A,B,C) &= \E \ba{\sumz t{T-1} \ve{\wh y_{t+1} - y_{t+1}}_2^2 - \sumz t{T-1}\ve{\wh y_{t+1,\KF} - y_{t+1}}_2^2}
\end{align}
where $\wh y_{t+1,\KF}$ is the prediction given by~\eqref{e:CKF}. The challenge to competing with the Kalman filter prediction is that the Kalman filter has memory: its prediction depends on a state estimate $x_t^-$ kept in memory. We can remove this dependence by ``unrolling" the Kalman filter and then truncating. Then we find that $\wh y_{t+1,\KF}$ is approximately a linear function of $u_{t-\ell+1},\ldots, u_t$ and $y_{t-\ell+1},\ldots, y_t$ for large enough $\ell$: letting $C_{\KF}=C$,
\begin{align*}
    \wh y_{t+1,\KF} &= Fu_{t:t-\ell+1} + Gy_{t:t-\ell+1}\\
    \text{where } F&= (C_{\KF}B_{\KF,u}, C_{\KF}A_{\KF}B_{\KF,u},\ldots, C_{\KF}A_{\KF}^{\ell-1}B_{\KF,u})\\
    \text{and }G&=(C_{\KF}B_{\KF,y}, C_{\KF}A_{\KF}B_{\KF,y},\ldots, C_{\KF}A_{\KF}^{\ell-1}B_{\KF,y}).
\end{align*}
In other words, we can approximate the Kalman filter with an autoregressive filter of length $\ell$.

The framework of online least-squares (Algorithm~\ref{a:orr}) now applies with 
$x_t\mapsfrom \scoltwo{u_{t:t-\ell+1}}{y_{t:t-\ell+1}}$, $y_t\mapsfrom y_{t+1}$, $A\mapsfrom (F,\;G)$, $n\mapsfrom n$, and $m\mapsfrom \ell(d+m)$, giving Algorithm~\ref{a:ols-ar}. We let $u_s=0$ and $y_s=0$ for $s<0$.

\subsection{Norms and sufficient length}

First we define the \emph{sufficient length} of a system. Roughly speaking, the sufficient length $R(\ep)$ is the length at which we can truncate a finite impulse response (FIR) filter, so that when inputs are bounded by 1, we incur at most $\ep$ prediction error at any time step. This notion was introduced by~\cite{tu2017non} in the one-dimensional setting. 

We first recall some concepts from control theory. In particular, the definition of sufficient length depends on the $\cal H_\iy$ norm.

\begin{df}
Let $F$ be a stable, linear time-invariant (LTI) system, represented as the transfer function $F(z) = \sumz j{\iy} F_j z^{-j}\in \R^{n\times m}[[z^{-1}]]$. (This is a matrix-valued Laurent series whose coefficients $(F_0,F_1,\ldots)$ form the impulse response function, that is, the response to input $v\in \R^{m}$ is $(F_0v, F_1v,\ldots)$.)
Define the $\cal H_\iy$ norm of $F$ to be 
$\ve{F}_{\cal H_\iy}:=\max_{|z|=1}\ve{F(z)}_2$. 
\end{df}


\begin{df}[{Sufficient length condition, \cite[Definition 1]{tu2017non}}]
\label{d:suff-len}
We say that a Laurent series (or LTI system) $F$ has stability radius $\rh\in (0,1)$ if $F$ converges for $\set{x\in \C}{|x|>\rh}$. 
Let $F$ be stable with stability radius $\rh\in (0,1)$. Fix $\ep>0$. Define the sufficient length
\begin{align}
R(\ep) &=\ce{
\inf_{\rh<\ga<1}\rc{1-\ga}\ln \pf{\ve{F(\ga z)}_\iy}{\ep(1-\ga)}
}.
\end{align}
\end{df}
Note that having a dependence on sufficient length is analogous to having a $\frac{1}{1-\rh(A)}$ dependence on the spectral radius of $A$, for learning a LDS. Roughly, if we ignore factors depending on condition numbers, for a LDS with dynamics given by $A$, the sufficient length $R(\ep)$ is on the order of $O\pa{\frac{1}{1-\rh(A)}\cdot \ln \prc{\ep}}$.

The following lemma says that if we are content with an error of $\ep$, we can safely truncate the impulse response function at length $R(\ep)$. Note that~\cite{tu2017non} give the proof for the one-dimensional case, but the same proof works in the multi-dimensional setting.
\begin{lem}[{\citealt[Lemma 4.1]{tu2017non}}]\llabel{l:suff-l}
Suppose $F$ is stable with stability radius $\rh\in (0,1)$. Then 
$\ve{F_{\ge L}}_1 := \sum_{k\ge L} \ve{F(k)} \le \max_{\rh<\ga<1} \fc{\ve{F(\ga z)}_\iy\ga^L}{1-\ga}.$
Hence, if $L\ge R(\ep)$, then $
\ve{F_{\ge L}}_1 \le \ep$. 
\end{lem}

If the LDS is not stable, then we cannot truncate the dependence on past inputs. The key observation 
is that even if the original LDS is not stable, the LDS defined by the Kalman filter is stable \citep{anderson2012optimal}. Hence, there is some sufficient length at which we can truncate the unrolled Kalman filter. 


\begin{thm}[{\cite[\S4.4]{anderson2012optimal}}] 
When the LDS~\eqref{e:lds1}--\eqref{e:lds2} 
satisfies Assumption~\ref{asm:obs}, the associated Kalman filter (as defined in Section~\ref{s:prelims-hidden}) is strictly stable: $\rh(A_{KF})<1$.
\end{thm}

\subsection{Proof of Theorem~\ref{t:kalman}}

\begin{proof}[Proof of Theorem~\ref{t:kalman}]
We first condition on all the noise terms $\ve{\xi_t}\le \ve{\Si_x}O\pa{\sqrt{d\ln \pf{T}{\de}}}=:C_\xi$, $\ve{\eta_t}\le \ve{\Si_y}O\pa{\sqrt{n\ln \pf{T}{\de}}}=:C_\eta$. Choosing the constants large enough, we can ensure this happens with probability $\ge 1-\fc{\de}3$. 

Let 
$v_t = (h_t;\ldots; h_{t-\ell+1};u_t;\ldots;u_{t-\ell+1})$
For convenience set $h_t=0$, $u_t=0$ for $t<0$. Then 
$(v_t')_{t=0}^{T}$ satisfies the following:
\begin{align}
v_t' &= A'v_{t-1}' + \xi_t'\\
A'&=\pa{\begin{array}{c|c}
\begin{array}{ccc|c}
A &  & & O\\
\hline I &  & & O \\
 & \ddots & &\vdots \\
 &  & I & O
\end{array} & \begin{array}{cccc}
B & O & \cdots & O \\
O & \ddots & & \vdots \\
\vdots &  & & \\
O & \cdots & & O\\
\end{array}\\
\hline  & \begin{array}{ccc|c}
O & \cdots & O & O \\
\hline I &  && O\\
 & \ddots && \vdots \\
 &  & I& O
\end{array}
\end{array}}\\
v_0' &=\pa{
\begin{array}{c}
h_0\\
\mathbf 0\\
\hline
u_0\\
\mathbf 0
\end{array}
}\in K_0':=B_{C_0}^d\opl \{0\}^{(\ell-1)d} \opl B_{C_u}^m\opl \{0\}^{(\ell-1)m}
\\
\xi_t'&=\pa{
\begin{array}{c}
\xi_t\\
\mathbf 0\\
\hline
u_t\\
\mathbf 0
\end{array}
}\in K':=B_{C_\xi}^d\opl \{0\}^{(\ell-1)d} \opl B_{C_u}^m\opl \{0\}^{(\ell-1)m}.
\end{align}
If $h_0=h$, $u_0=u$, and $u_t=0$ for $t\ge 1$, and there is no noise, then $h_t=A^th+A^{t-1}Bu$ for $t\ge 1$. Hence
\begin{align}
\llabel{e:A'k'}
\ve{(A')^{k'} \ctzz{h}{u}}_2 &\le \sqrt{\ell}
\ba{\max_{0\le k\le k'} f_A(k)\ve{h} +\pa{ \max_{0\le k\le k'-1} f_A(k) C_B+1}\ve{u}}.
\end{align}
Note also that
\begin{align}
\llabel{e:A'k'2}
    \max_{v\in B_r^{\ell(d+m)}}\ve{ (A')^{k'} v}
    &= \max_{(h;u)\in B_r^{d+m}} \ve{(A')^{k'} \ctzz hu}
\end{align}
because we can check that
\begin{align*}
    \ve{(A')^{k'} \pa{
\begin{array}{c}
h_{\ell-1}\\
\vdots\\
h_0\\
\hline
u_{\ell-1}\\
\vdots\\
u_0
\end{array}
}}
&\le 
\ve{(A')^{k'}
\ctzz{\fc{\ve{h_{\ell-1:0}}}{\ve{h_{\ell-1}}}h_{\ell-1}}
{\fc{\ve{u_{\ell-1:0}}}{\ve{u_{\ell-1}}}u_{\ell-1}}}
\end{align*}
We will apply Lemma~\ref{l:xt-l1-comb}.
Let $k'=\kldm$ and $C'=C_AC_0+C_AC_BC_u+C_u$.
We bound~\eqref{e:Lvg}. By~\eqref{e:A'k'},~\eqref{e:A'k'2}, and Lemma~\ref{l:J-power},
\begin{align*}
f_{A',K_0'\cup B_{\Laldm}^{\ell(d+m)}}(k)
&= \sqrt{\ell} \Big[\max_{0\le k\le k'}
f_A(k) \pa{C_0+\Laldm}\\
&\quad 
+ \pa{\max_{0\le k\le k'-1}f_{A}(k)+1}\pa{C_u + \Laldm}
\Big]\\
&\le 
\sqrt{\ell}\Big[
C_A(k'+1)^{r-1}\pa{C_0+\Laldm} \\
&\quad + (C_AC_B(k')^{r-1} + 1) \pa{C_u + \Laldm}
\Big]\\
&\le \sqrt{\ell}\ba{C'(k'+1)^{r-1} + (C_A(C_B+1)+1)\Laldm}\\
\sumz k{k'-1} f_{A',K'}(k)
&\le \sqrt{\ell}\ba{
\sumz k{k'-1}\max_{0\le j\le k}  f_A(k) C_\xi + 
\pa{\sumz k{k'-2}\max_{0\le j\le k}  f_A(k)C_B+1}C_u
}\\
&\le \sqrt{\ell}[(k')^r C_\xi + ((k'-1)^rC_B+1)C_u].
\end{align*}
By Lemma~\ref{l:xt-l1-comb}, for $0\le t\le T$, 
there exist $a_s\in \R$ and $v\in \R^{\ell(d+m)}$ such that 
$v_t=\sumz s{t-1}a_s v_s + v$ with 
\begin{align*}
\sumz t{T-1}|a_t|&\le \Laldm=:L_a'\\
\ve{v} &\le 
\max_{0\le k\le k'} f_{A',K_0'\cup B_{\fc{2}{\ln 2}\ell (d+m)}^{\ell(d+m)}}(k) +  \pa{\sumz k{k'-1}f_{A',K'}(k)} \pa{\Laldm+1}\\
&=\sqrt{\ell}\Big[C'(k'+1)^{r-1} + (C_A(C_B+1)+1+(k')^r C_\xi \\
&\quad + ((k'-1)^rC_B+1)C_u)\pa{\Laldm+1}\Big]
=:L_v'
\end{align*}
Now let 
$v_t=(y_t;\cdots; y_{t-\ell+1};u_t;\cdots ;u_{t-\ell+1})$.
Then
\begin{align*}
v_t &=C'v_t' + \coltwo{\eta_{t:t-\ell+1}}{\mathbf0}\\
\text{where }
C'&=\pa{\begin{array}{c|c}
I_\ell\ot C & O\\
\hline
O &I_{\ell m}
\end{array}}.
\end{align*}
Now if $v_t'=\sumz s{t-1}a_sv_s'+v$ with $\sumz s{t-1} a_s\le L_a$ and $\ve{v}\le L_v$, then 
\begin{align*}
C'v_t' &= \sumz s{t-1} a_sC'v_s' + C'v\\
\implies
v_t - \coltwo{\eta_{t:t-\ell+1}}{\mathbf 0}
&= \sumz s{t-1} a_s \pa{
v_s - \coltwo{\eta_{s:s-\ell+1}}{\mathbf 0}
}+C'v\\
\implies
v_t &= \sumz s{t-1} a_sv_s + \coltwo{\eta_{s:s-\ell+1}}{\mathbf 0} - \sumz s{t-1}a_s\coltwo{\eta_{s:s-\ell+1}}{\mathbf0}+C'v
\end{align*}
so it can be written as a linear combination of previous $v_s$'s with 
\begin{align*}
L_a&=L_a'\\
L_v&=C_C L_v' + (L_a+1) \sqrt \ell C_\eta 
\end{align*}
If $x_0\sim N(0,\Si_{\KF,x})$, then $x_0^-=0$, the steady-state Kalman filter applies, and unfolding the Kalman filter recurrence gives
\begin{align*}
y_{t+1}&= \sumz st F_s u_{t-s} + \sumz st G_s y_{t-s} +\cancel{C_{\KF}A_{\KF}^{t+1} x_0^-} +\ze_{t+1}
\end{align*}
where $\ze_{t+1}|\cal F_t\sim N(0,\Si_{\KF,y})$ is $\ve{\Si_{\KF,y}}^2$-subgaussian.
Then $(y_t)$ satisfies Lemma~\ref{l:stoch} with $L_a$, $L_v$, and with $w_t=y_{t+1}$ where
\begin{align*}
y_{t+1} &= (F,\; G) \coltwo{u_{t:t-\ell+1}}{y_{t:t-\ell+1}}+\ep_{t+1}+\ze_{t+1}\\
\ep_{t+1} &= \sum_{s=\ell}^{t}(F_s u_{t-s} + G_sy_{t-s}).
\end{align*}
By Lemma~\ref{l:Mbd}, 
\begin{align*}
    \max_{0\le s\le t-\ell} \max\{\ve{u_{t}},\ve{y_t}\}
    &\le 
    T^r C_C C_A (C_BC_u + C_\xi)=:K.
\end{align*}
By choice of $\ell=R(\ep')$, where $\ep'=\fc{\ep}{K}$, 
we get 
\begin{align*}
\ve{\ep_{t+1}}&=
\ve{\sum_{s=\ell}^{t} (F_su_{t-s} + G_sy_{t-s})}\le \ep' \max_{0\le s\le t-\ell} \max\{\ve{u_{t}},\ve{y_t}\}\le \ep.
\end{align*}
We also know $\xi_{t+1}|\cF_t\sim N(0,\Si_{\KF,y})$. Apply Lemma~\ref{l:stoch} to get a polynomial bound on\\ $\max_{0\le t\le T-1} \ve{F_tu_{t:t-\ell+1}+G_ty_{t:t-\ell+1} - y_{t+1}}$ and Theorem~\ref{t:orr} to finish.

\end{proof}

\section{OLS regret bound}
\label{a:ols-regret}
We show Theorem~\ref{t:orr}.

We note that~\cite{cesa2006prediction} show the case where $A\in \R^{1\times m}$; the case for $A\in \R^{n\times m}$ essentially follows the same proof.
Note that the objective function for $A$ decomposes as a sum of objective functions for each row $A_i$:
\begin{align*}
\mu \ve{A}_F^2 + \sumo s{t-1} \ve{Ax_s-y_s}^2
&= 
\sumo in \pa{\mu \ve{A_i}^2 + \sumo s{t-1} \ve{A_i x_s-(y_s)_i}^2}.
\end{align*}
Hence, running Algorithm~\ref{a:orr} is equivalent to running the algorithm on each coordinate of $y_t\in \R^n$ separately.

Note, however, that if we use \cite[Thm. 11.7]{cesa2006prediction} as a black box, and apply it to every row of $A_t$, we get the bound
\begin{align*}
    R_T(A) 
    &\le \mu\ve{A}_F^2 + 
    \sumo in \max_{1\le t\le T} \pa{(y_t)_i - (A_{t})_ix_t}_2^2 m \ln \pa{1+\fc{TM^2}{m}}\\
    &\le 
    \mu\ve{A}_F^2 + 
     \max_{1\le t\le T} \ve{y_t - A_tx_t}_2^2 mn \ln \pa{1+\fc{TM^2}{m}}
\end{align*}
where the second inequality follows from using the naive bound $\pa{(y_t)_i - (A_{t})_ix_t}_2^2 \le \ve{y_t - A_tx_t}_2^2$ and has an extra factor of $m$.

We refer to~\cite{cesa2006prediction} for the notation we will use. 

\begin{proof}[Proof of Theorem~\ref{t:orr}]
Define the objective function $\Phi_t^*:\R^{n\times m}\to \R$ by
\begin{align*}
    \Phi_t^*(A) &= \mu\ve{A}_F^2 + \sumo st \ve{Ax_s-y_s}^2
\end{align*}
and the component objective functions $\Phi_{t,i}^*:\R^{1\times m}\to \R$ (for $1\le i\le n$) by
\begin{align*}
    \Phi_{t,i}^*(a) &= \mu \ve{a}^2 + \sumo st \ve{ax_s-(y_s)_i}^2
\end{align*}
so that 
\begin{align*}
    \Phi_t^*(A) &= \sumo in \Phi_{t,i}^*(A_i).
\end{align*}
Hence the Bregman divergence also decomposes.
Let $\Si_t = \mu I + \sumo st x_sx_s^\top$.
Using the calculation in \cite[pg. 319]{cesa2006prediction}
\begin{align*}
    D_{\Phi_t^*}(A_{t-1},A_t)
    &=\sumo in D_{\Phi_{t,i}^*}((A_{t-1})_i, (A_t)_i)\\
    &=\sumo in ((A_{t-1})_i x_t - (y_t)_i)^2 x_t^\top \Si_t^{-1}x_t\\
    &=\ve{A_{t-1}x_t-y_t}^2 x_t^\top \Si_t^{-1}x_t.
\end{align*}
The rest of the proof then follows the proof of Theorem 11.7, together with the remark following the proof.
The only difference is that $\ve{w_{t-1}^\top x_t-y_t}^2$ has been replaced by $\ve{A_{t-1}x_t-y_t}^2$.
Note also that they take $\mu=1$, but the calculations go through with arbitrary $\mu>0$.

\end{proof}

\section{Alternate approach to proving anomaly-freeness}
\label{s:php}
\label{a:php}

In this section we give an alternate proof of anomaly-freeness, which we discovered after the first draft of the paper. This approach does not require the number-theoretic lemma, Lemma~\ref{l:prime-prod}. However, the polynomial bounds are worse than the proof via the volume doubling argument.

They key fact is that the characteristic polynomial has a multiple of not-too-large degree with small coefficients. Then we can apply the same approach as Lemma~\ref{l:ch}. The following lemma is an adaptation of \cite[Lemma 5.4]{chen2016fourier}.

\begin{lem} 
\label{l:poly-small-multiple}
For any $k\ge 1$ and any $z_1,\ldots, z_k\in \C$ of absolute value at most 1, there exists a degree $n=O(k^2\ln k)$ polynomial $P(z) = \sumz jn c_j z^j$ with the following properties:
\begin{align*}
    \prodo ik (z-z_i) &\mid 
    P(z) \\
    c_n &= 1\\
    |c_j|&\le 11,\quad \forall j\in \{1,\ldots, n\}.
\end{align*}
\end{lem}

\begin{lem}[Claim A.4, \cite{chen2016fourier}]
\label{l:php}
Let $z_1,\ldots, z_k\in \C$ have absolute value at most 1, and let $Q(z) = \prodo ik (z-z_i)$. 
For any $m=\Om(k^2\ln k)$
and $P^*(z)= \sumz im \al_i z^i$ with coefficients $|\al_i|\le 10$ for any $i\in \{0,1,\ldots, m\}$, such that every coefficient of $P^*(z)\bmod Q(z)$ is bounded by $2^{-m/k}$.
\end{lem}

Note this is a corrected statement of Claim A.4. (The original statement erroneously had $2^{-m}$ instead of $2^{-m/k}$.) The proof is an elegant and delightful application of the pigeonhole principle. Note that \cite{chen2016fourier} state the theorem for $z_1,\ldots, z_k$ on the unit circle, but the exact same proof goes through if they are allowed to be inside the unit circle. \\

\begin{proof}[Proof of Lemma~\ref{l:poly-small-multiple}]
Take $m=\Te(k^2\ln k)$, so that $2^{-m/k}\le \ep:=\rc{12}$, and so that Lemma~\ref{l:php} applies. 
Let $P^*(z) =\sumz in \al_i z^i$ be as in Lemma~\ref{l:php}, with $n\le m$ and nonzero leading coefficient $\al_n\ne 0$, and let $r(z) =  P^*(z)\bmod Q(z) = \sumz i{k-1}\ga_iz^i$. Set $\ga_i=0$ for $i\ge k$. 
Because $r(z)$ is the residue, $P^*(z)-r(z)$ is divisible by $\prodo ik (z-z_i)$. 

Let 
$$P(z) = \fc{P^*(z)-r(z)}{\al_n-\ga_n}.$$
Note that $P(z)$ is monic, and every coefficient is bounded in absolute value by $\fc{10+\ep}{1-\ep}\le 11$. This is the desired polynomial.
\end{proof}

Using the multiple of the characteristic polynomial given by Lemma~\ref{l:poly-small-multiple} instead of the characteristic polynomial in Lemma~\ref{l:ch} gives the following result on anomaly-freeness.

\begin{lem}
Given Assumptions~\ref{asm:1} with $B=O$, suppose that $A\in \R^{d\times d}$ is diagonalizable as $A=VDV^{-1}$ where $\ve{V}\ve{V^{-1}}\le C_A$.
Suppose that $|w^\top x_t|=\Om(d^2(\ln d)C_AC_\xi)$ and $t=\Om(d^2\ln d)$. 
Then for any unit vector $w\in\R^d$, 
there exist $O\pa{\min\bc{\fc{|w^\top x_t|}{d^4(\ln d)^2 C_AC_\xi}, \fc{t}{d^4(\ln d)^2}}}$ values of $s$, $0\le s\le t-1$, such that 
\begin{align*}
|w^\top x_s| &\ge \Om\pf{|w^\top x_t|}{d^2\ln d}.
\end{align*}
\end{lem}

\begin{proof}
Because $\rh(A)\le 1$, all zeros of the characteristic polynomial of $A$ have absolute value $\le 1$. By Lemma~\ref{l:poly-small-multiple}, there exists a multiple $p(x) = \sumz im a_i^{(1)}x^i$ of the characteristic polynomial of $A$ such that:
\begin{itemize}
    \item $p$ has degree $n=O(k\ln k)$. 
    \item $p$ is monic ($a_n^{(1)}=1$).
    \item All coefficients are bounded: $|a_i^{(1)}|\le 11$. 
\end{itemize}
Then by unfolding the recurrence as in~\eqref{e:ch-recur}, we obtain
\begin{align*}
    x_t &=\sum_{i=0}^{n-1}-a_i^{(1)} x_{t-n+i} \ub{+ \sum_{i=0}^{n-1}a_i^{(1)}\sum_{\tau=1}^{i} A^{i-\tau} \xi_{t-n+\tau} + \sum_{\tau=1}^n A^{d-\tau}\xi_{t-n+\tau}}{=:v^{(1)}}
\end{align*}
Note that $\ve{A^i\xi}\le \ve{VD^iV^{-1}}\ve{\xi}\le C_A \ve{\xi} \leq C_AC_\xi$. 
We write $x_t = -\sum_{i=1}^{n-1}a_i^{(1)} x_{t-n+i} + v^{(1)}$ where
\begin{align*}
    \|v^{(1)}\|
    &= \left\| \sum_{i=0}^{n-1}a_i^{(1)}\sum_{\tau=1}^{i} A^{i-\tau} \xi_{t-n+\tau} + \sum_{\tau=1}^n A^{n-\tau}\xi_{t-n+\tau} \right\|
    \\ 
    &\leq \sum_{i=0}^{n-1}|a_i^{(1)}| \sum_{\tau=1}^{i} 
    C_A
    \|\xi_{t-n+\tau}\|
    + \sum_{\tau=1}^{n} 
    C_A
    \|\xi_{t-n+\tau}\|
    \\
    &\leq 
    11(n-1)C_AC_\xi + nC_AC_\xi\le 12nC_AC_\xi
\end{align*}
If $w\in \R^d$ is a unit vector such that $|w^\top x_t|\ge 24nC_AC_\xi$, we have $|w^\top x_t| \ge 2\ve{v^{(1)}}$. Noting that 
$w^\top x_t = \sumz i{n-1} - a_i^{(1)} x_{t-n+i} + v^{(1)}$, by Lemma~\ref{c:exists-large} there exists an index $0\le i\le n-1$ such that 
\begin{align*}
    |w^\top x_{t-n+i}| \geq \frac{|w^\top x_{t}| - \|v^{(1)}\|
    }{\sum_{i=0}^{n-1} |a_i^{(1)}|} \ge 
    \fc{|w^\top x_t|}{2\sumz i{n-1} |a_i^{(1)}|}
    \ge \fc{|w^\top x_t|}{22n}
\end{align*}
In order to obtain many large past $x$'s, we apply the same argument on sequences $x_t, x_{t-k}, \ldots, x_{t-kd}$ by considering the recurrence
\begin{align*}
    x_t &= A' x_{t-k} + \xi_t^{(k)}\\
    A'&=A^k\\
    \xi_t^{(k)} &= \sumo jk A^{k-j} \xi_{t-k+j}.
\end{align*}
Let $p_k(x)=\sumz in a_{i}^{(k)} x^i$ be a multiple of the characteristic polynomial of $A^k$ given by Lemma~\ref{l:poly-small-multiple}, for some $n=O(d^2\ln d)$. 
We obtain $x_t=-\sum_{i=1}^{n-1}a_i^{(k)} x_{t-k(n-i)}+v^{(k)}$ with 
\begin{align}
    \|v^{(k)}\|
    &= \left\| \sum_{i=0}^{n-1}a_i^{(k)}\sum_{\tau=1}^{i} A^{\prime i-\tau} \xi_{t-nk+\tau k}^{(k)} + \sum_{\tau=1}^n A^{\prime  n-\tau}\xi_{t-nk+\tau k}^{(k)} \right\|
    \\ 
    &\leq 
    11k(n-1)C_AC_\xi + knC_AC_\xi\le  12knC_AC_\xi
\end{align}
and $\sumz i{n-1} |a_i^{(k)}|\leq 11n$.
We then pick $k=1, 2, \ldots, O\pa{\min\bc{\fc{|w^\top x_t|}{d^2(\ln d) C_AC_\xi},\fc{t}{d^2\ln d} }}$. For each choice of $k$, we know by design that $|w^\top x_t|\geq 2\|v^{(k)}\|_2$, and therefore there must exist an $x$ in the sequence $x_{t-k}, \ldots, x_{t-kd}$ such that $|w^\top x| = \Om\pf{|w^\top x_t|}{d^2\ln d}$. In this way, we are able to collect in total $L=O\pa{\min\bc{\fc{|w^\top x_t|}{d^2(\ln d) C_AC_\xi},\fc{t}{d^2\ln d} }}$ many such $x$'s. 
To finish the argument, we note that out of the $L$ collected $x$'s, there are at least $O\pf{L}{d^2\ln d}$ distinct ones, since one $x$ can appear in at most $O(d^2\ln d)$ different sequences.
\end{proof}

\section{Open Questions}\label{a:future}
Several fundamental questions come to mind:
\begin{enumerate}
    \item \emph{Is the $T^{\frac{2r+1}{2r+2}}$ rate optimal?} Even in the diagonalizable ($r=1$) case, this is unresolved.
    \item \emph{What is the rate for partially observed LDS when the noise is not Gaussian?}
    We stated Theorem~\ref{t:kalman} for Gaussian noise only, but a similar result will hold as long as at steady state, $\E[y_t|\cal F_{t-1}]$ is given by a linear function of $y_{t-1:0}$, $u_{t-1:0}$, and the estimated state $x_0^-$. This is required in order for the random variable $y_t|\cal F_{t-1}$ to be a linear function of past observations and inputs, plus a random variable $\ze_t$ \emph{with zero mean}. In general, if the noise $\xi_t$ is not Gaussian, then $\ze_t$ is not zero-mean (even if $\xi_t$ is zero-mean). We use the same machinery as in the proof of Theorem~\ref{t:main-stoch} to conclude Theorem~\ref{t:kalman}, so our proof strategy cannot handle arbitrary zero-mean noise $\xi_t$. 
    
For non-Gaussian zero-mean noise $\xi_t$, we can instead treat the $\ze_t$ as adversarial noise, and use the machinery behind Theorem~\ref{t:main} to obtain a $T^{\fc{2r+1}{2r+2}}$ regret bound. It is an interesting question whether we can obtain polylogarithmic regret with respect to the best linear filter in this case.
\item \emph{Can we obtain bounds depending on system order $d$ rather than rollout length $\ell$?} Theorem~\ref{t:kalman} depends polynomially on the sufficient rollout length $\ell$, rather than the intrinsic dimensionality of the problem given by $d$ amd $m$. This seems to be a limitation of using the improper autoregressive approach; can we do better using techniques from system identification?
\end{enumerate}
We leave these for future work.


\end{document}